\documentclass[letterpaper]{article} 
\usepackage{aaai24}  
\usepackage{times}  
\usepackage{helvet}  
\usepackage{courier}  
\usepackage[hyphens]{url}  
\usepackage{graphicx} 
\usepackage{natbib}  
\usepackage{caption} 
\usepackage{algorithm}
\usepackage{algorithmic}
\usepackage{tabularx}
\usepackage{newfloat}
\usepackage{listings}
\DeclareCaptionStyle{ruled}{labelfont=normalfont,labelsep=colon,strut=off} 
\floatstyle{ruled}
\newfloat{listing}{tb}{lst}{}
\floatname{listing}{Listing}
\usepackage{amsmath,amsfonts,bm}
\usepackage{booktabs} 

\newcommand{\RNum}[1]{\uppercase\expandafter{\romannumeral #1\relax}}
\newcommand{\Rnum}[1]{\lowercase\expandafter{\romannumeral #1\relax}}

\def\Figref#1{Figure~\ref{#1}}

\def\Tabref#1{Table~\ref{#1}}

\def\eqref#1{equation~(\ref{#1})}
\def\Eqref#1{Equation~(\ref{#1})}

\def\Algref#1{Algorithm~\ref{#1}}

\def\peqref#1{(\ref{#1})}

\def \Thmref#1{Theorem~\ref{#1}}

\def\0{\bm{0}} 
\def\1{\bm{1}}

\DeclareMathAlphabet{\mathsfit}{\encodingdefault}{\sfdefault}{m}{sl}
\SetMathAlphabet{\mathsfit}{bold}{\encodingdefault}{\sfdefault}{bx}{n}

\usepackage{kotex}
\usepackage{xcolor}
\usepackage{soulutf8}
\usepackage{amsthm}
\usepackage[english]{babel}
\newtheorem{theorem}{Theorem}
\usepackage{amsthm}

\newtheorem{definition}{Definition}
\newtheorem*{T1}{Theorem~\ref{thm:thm2}}
\newtheorem*{T2}{Theorem~\ref{thm:thm_fair}}
\usepackage{amssymb}
\usepackage[algo2e]{algorithm2e}
\usepackage{comment}
\usepackage{subfigure}
\usepackage{amsmath}
\usepackage{multirow}
\usepackage{dsfont}

\title{Fair Sampling in Diffusion Models through Switching Mechanism}

\author {
    Yujin Choi\textsuperscript{\rm 1}\equalcontrib,
    Jinseong Park\textsuperscript{\rm 1}\equalcontrib,
    Hoki Kim\textsuperscript{\rm 1},
    Jaewook Lee\textsuperscript{\rm 1},
    Saerom Park\textsuperscript{\rm 2}\thanks{Corresponding author.}
}
\affiliations {
    \textsuperscript{\rm 1}Seoul National University\\
    \textsuperscript{\rm 2}Ulsan National Institute of Science and Technology (UNIST)\\
    uznhigh@snu.ac.kr, jinseong@snu.ac.kr, ghrl9613@snu.ac.kr,
    jaewook@snu.ac.kr,
    srompark@unist.ac.kr
}

\usepackage{bibentry}

\begin{document}

\maketitle

\begin{abstract}
Diffusion models have shown their effectiveness in generation tasks by well-approximating the underlying probability distribution. However, diffusion models are known to suffer from an amplified inherent bias from the training data in terms of fairness. While the sampling process of diffusion models can be controlled by conditional guidance, previous works have attempted to find empirical guidance to achieve quantitative fairness. 
To address this limitation, we propose a fairness-aware sampling method called \textit{attribute switching} mechanism for diffusion models. Without additional training, the proposed sampling can obfuscate sensitive attributes in generated data without relying on classifiers.
We mathematically prove and experimentally demonstrate the effectiveness of the proposed method on two key aspects: (i) the generation of fair data and (ii) the preservation of the utility of the generated data.
\end{abstract}

\noindent 
\section{Introduction}
Generative models have shown promise across diverse domains to generate unseen data examples. Recently, diffusion-based models \cite{song2020score,rombach2022high} have shown the effectiveness in various fields, including image \cite{rombach2022high}, video~\cite{blattmann2023align}, and audio~\cite{kong2020diffwave} generation. However, diffusion models also suffer from generating data possessing inherited and amplified bias from the training dataset \cite{friedrich2023fair}, similar to the problems observed in other using generative models such as autoencoder \cite{amini2019uncovering} or generative adversarial networks (GAN) \cite{sattigeri2019fairness,xu2018fairgan}. 
Since the diffusion model needs specific methods to generate synthetic data by denoising from stochastic differential equation (SDE) and ordinary differential equation (ODE)  \cite{song2020denoising}, it requires unique methods for achieving fairness in contrast to previous fair generative models \cite{kenfack2021fairness}.
In response, recent studies have investigated fairness issues in diffusion models \cite{friedrich2023fair,sinha2021d2c}. These studies have mainly focused on quantitative aspects of achieving fairness within the diffusion model framework.  
For instance, \citet{friedrich2023fair} introduces the fair guidance to generate sample fair data to resolve the imbalanced sampling for certain classes. 
However, assessing models through a specific classifier falls short of capturing the broader performance spectrum~\cite{xu2019fairgan+}. In addition, evaluating the fairness of a generative model through classification tasks introduces complexities in distinguishing whether the observed discrimination stems from the generator or the classifier. Instead, in this study, we employ the concept of \textit{$\epsilon$-fairness}, as introduced in \cite{feldman2015certifying, xu2018fairgan}. We aim to evaluate data fairness concerning the sensitive attribute itself without a specific label by achieving data fairness in diffusion models. 

To achieve the goal, we present a sampling method called \textit{attribute switching} to generate synthetic data under the $\epsilon$-fairness framework. Inspired by the finding in \cite{choi2022perception} that diffusion models learn different characteristics at each step, our work aims to generate data that is both independent of sensitive attributes and satisfies utility requirements. Subsequently, we prove the theoretical condition of the transition point, ensuring independence so that achieving fairness. Furthermore, our model does not require additional training on top of standard diffusion models, offering flexibility in its applicability to various pre-trained models. As our method is solely designed for diffusion sampling, we demonstrate the effectiveness of attribute switching using various recent diffusion models.
We summarize our main contributions as follows:
\begin{itemize}
    \item We address the data fairness issue in diffusion models, aiming to ensure sensitive attribute independence without requiring additional classifiers. To the best of our knowledge, this is the first attempt to investigate $\epsilon$-fairness within diffusion models.

    \item We introduce a novel sampling approach, which we call \textit{attribute switching}, that can generate fair synthetic data by matching the data distribution. Additionally, we mathematically prove its effectiveness and present an efficient algorithm to find the optimal transition point to achieve fairness.
    
    \item We show the effectiveness of the proposed fair sampling on various datasets including the state-of-art diffusion models without additional training. 
\end{itemize}

\begin{table*}[!ht]
\centering
\resizebox{\textwidth}{!}{%
    \begin{tabular}{c|cc|cc}
    \toprule
     & \multicolumn{2}{c|}{Classification Fairness} & \multicolumn{2}{c}{Data Fairness} \\ \hline
    Notions & \multicolumn{1}{c|}{Statistical parity} & Equalized odds & \multicolumn{1}{c|}{Statistical fairness} & $\epsilon$-fairness \\ \hline
    Equations & \multicolumn{1}{c|}{\begin{tabular}[c]{@{}c@{}}$P(\eta(X)|S = 1)$\\ $=P(\eta(X)|S = 0)$\end{tabular}} & \begin{tabular}[c]{@{}c@{}}$P(\eta(X)=1|Y, S = 1)$\\ $=P(\eta(X)=1|Y,S = 0)$\end{tabular} & \multicolumn{1}{c|}{\begin{tabular}[c]{@{}c@{}}$P(Y|S = 1)$\\ $= P(Y|S = 0)$\end{tabular}} &   \multicolumn{1}{c}{\begin{tabular}[c]{@{}c@{}}\Eqref{eq:ber}\\ BER$(f(X),S)>\epsilon$\end{tabular}}  \\ \hline 
    Independent to Classifier $\eta$ & \multicolumn{1}{c|}{X} & X & \multicolumn{1}{c|}{O} & O \\ \hline
    Independent to Label $Y$ & \multicolumn{1}{c|}{X} & X & \multicolumn{1}{c|}{X} & O
    \\\bottomrule
    \end{tabular}
}
\caption{Various notions for fairness. Among the fair notions, $\epsilon$-fairness does not rely on classifier $\eta$ and label $Y$.}
\label{tab:fairness_notions}
\end{table*}

\section{Preliminary and Related Work}


In this section, we introduce key notations, definitions, and related works to explain fairness issues in generative models. 

\paragraph{Notations} We consider the sensitive attributes $S \in \mathcal{S}$ and input data $X \in \mathcal{X}$, where $(X,S) \sim p_D$ represents the underlying distribution of the training data. For synthetic data, we introduce a hat notation $\mathcal{\hat{X}}$. In addition, when defining fairness notions for the classification tasks, we consider the label $Y \in \mathcal{Y}$ or a classifier $\eta : \mathcal{X} \rightarrow \mathcal{Y}$, where data distribution becomes $(X,Y,S) \sim p_{D}$. For diffusion models, $p_t(X_t)$ denotes the distribution of $X_t$ for the diffusion process $\{X_t\}_{t=0}^T$ and $X_0 \sim p_D(X)$.  
\subsection{Fairness Notions}
With regards to group fairness, prior research has introduced a range of fairness concepts \cite{feldman2015certifying, verma2018fairness}.  Among these, \Tabref{tab:fairness_notions} summarizes the predominant definitions~\cite{verma2018fairness}, which categorizes the notions into classification fairness (prediction-dependent) and data fairness (prediction-independent) categories. 
Among them, classification fairness requires the consideration of downstream tasks with the corresponding classifiers. For instance, the \textit{statistical parity} focuses on ensuring that predicted labels are independent of the sensitive attribute~\cite{dwork2012fairness}, and \textit{equalized odds} fairness aims that the classifier $\eta$ maintains the equal values of ROC for label $Y$~\cite{hardt2016equality}.

Therefore, to circumvent the necessity for an additional classifier, we focus on the concept of data fairness.  However, although \textit{statistical fairness} does not directly depend on the classifier $\eta$, it requires taking into account a specific downstream task based on the relationship between the sensitive attribute $S$ and the label $Y$.
Intuitively, it is more natural to consider the input data $X$ alongside the sensitive attribute $S$ to assess generative models. Thus, we focus on  \textit{$\epsilon$-fairness} \cite{feldman2015certifying, xu2018fairgan} defined as follows:

\begin{definition}
\textbf{(\boldsymbol{$\epsilon$}-fairness) }
For any function $f:\mathcal{X}\rightarrow\mathcal{S}$, a dataset $D = (X,S)$ satisfies $\epsilon$-fairness if 
$$
\textnormal{BER}(f(X),S)>\epsilon,
$$
where the balanced error rate (BER) is defined as  
\begin{equation}
\label{eq:ber}
    \small{\textnormal{BER}(f(X),S)=\frac{P(f(X)=0|S = 1)+P(f(X)=1|S = 0)}{2}.}
\end{equation}
\end{definition}
The $\epsilon$-fairness evaluates fairness concerning the sensitive attribute $S$ itself.
\citet{xu2018fairgan} pointed out that $\epsilon$-fairness can be measured by the classifiers trained on synthetic data. With a high classification error rate on real data, sensitive attribute $S$ is not predictable by synthetic data $\hat{X}$ and the inherent disparate impact in the real data can be removed in synthetic data.
In this context, the $\epsilon$-fairness implies that the synthetic data $\hat{X}$ from the fair generator can assure the independence of $S$, denoted as $\hat{X}\perp\!\!\!\perp S$.



\subsection{Diffusion Models}
Diffusion models \cite{ho2020denoising, dhariwal2021diffusion} become the de-facto standard method for generative models by achieving both sample qualities and diversity. To approximate the data distribution, diffusion models forward sample data into noise distribution from time steps $t=0$ to $t=T$, where $T$ is the number of steps in the denoising process. Then, the model learns the reverse diffusion process by matching score functions via SDE and ODE \cite{song2020denoising, song2020score}.
The forward diffusion process is described with an Itô SDE:
\begin{equation}
\label{eq:forward}
    dX_t = f(X_t, t)dt + g(t)dW_t,
\end{equation}
for drift coefficient $f: \mathcal{X} \times[0, T] \rightarrow \mathcal{X}$, diffusion coefficient $g:[0, T] \rightarrow \mathbb{R}$, and Brownian motion $W$.
For the forward diffusion, corresponding reverse SDE \cite{anderson1982reverse} conditioned on $S$ \cite{ho2021classifierfree} can be written as follows:
\begin{equation}
\label{eq:reverse}
    d\bar{X_t} = (f(\bar{X_t}, t) - g^2(t)\nabla_x\log p_t(\bar{X_t}|S))dt - g(t)d\bar{W_t}.
\end{equation}
\citet{song2020score} showed that this reverse SDE solution is equivalent to the following reverse ODE:
\begin{equation}
\label{eq:reverse_ode}
    d\bar{X_t} = (f(\bar{X_t}, t) - \frac{1}{2}g^2(t)\nabla_x\log p_t(\bar{X_t}|S))dt.
\end{equation}
To learn the score function via neural networks $\theta$, diffusion models minimize the following loss: 
\begin{equation}
\label{eq:loss_scorematching}
     \mathcal{L}(\theta)=\mathbb{E}_{X_t, t}[\|\psi_{\theta}(X_t, t, S)-\nabla_{X_t} \log p(X_t|S)\|_{2}^{2}],
\end{equation}
where $\psi_{\theta}(X_t, t, S)$ is a score function to learn. 

\begin{figure*}[!ht]
\centering     
    \subfigure[Diffusion stages: coarse,content, and cleaning]{\label{fig:snr}\includegraphics[width=70mm]{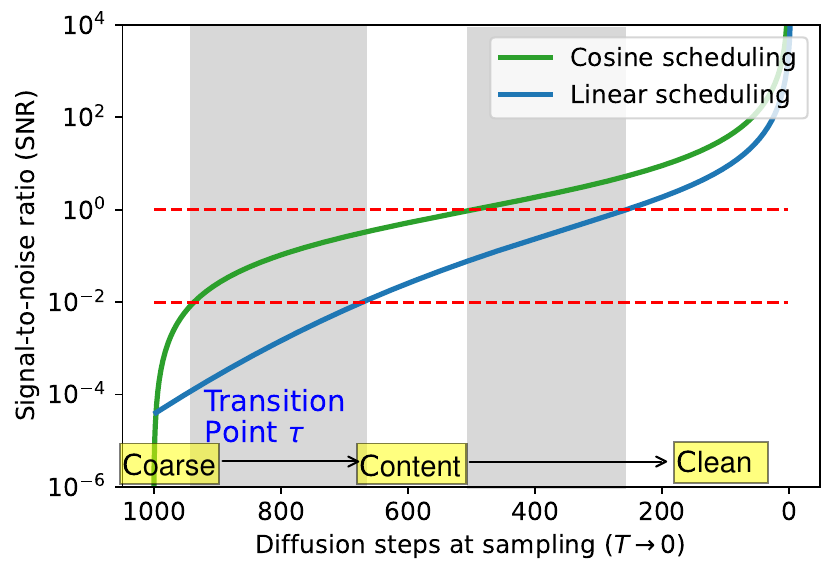}}
    \subfigure[Illustration of attribute switching mechanism (from $s_0$ to $s_1$)]{\label{fig:proposed_method_figure}\includegraphics[width=105mm]{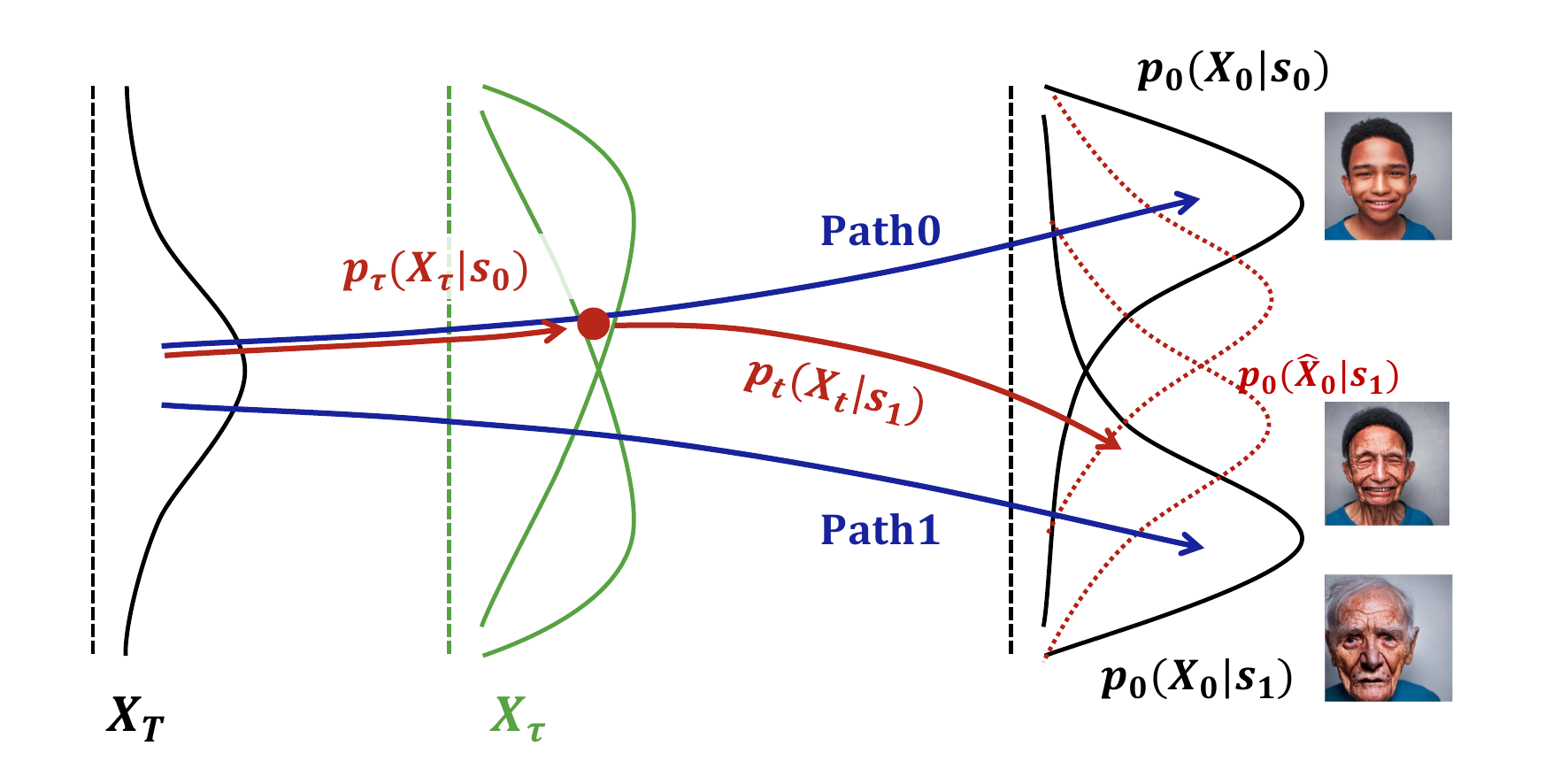}}
    \caption{Motivation concerning diffusion learning stages and illustration of the proposed method. In transition point $\tau$, the proposed method switches the condition of sensitive attributes from $s_0$ to $s_1$ to generate synthetic data satisfying $\epsilon$-fairness.}
    \label{fig:proposed_method}
\end{figure*}

\paragraph{Recent Diffusion Sampling} 
A distinct advantage of the diffusion model lies in its capacity to generate data without needing extra conditioning architectures \cite{ho2021classifierfree}. Thus, previous studies have investigated exploring flexible guidance sampling, such as leveraging text embedding \cite{rombach2022high,ramesh2022hierarchical} or designing new sampling methods using pre-trained models \cite{brack2023sega}. Different sampling methods can vary outcomes in terms of both sampling quality \cite{humayun2022polarity} and speed \cite{lu2022dpm}.

\subsection{Fair Generative Model} 
When considering downstream fairness, a common approach involves incorporating fairness loss into the training objective. FairGAN(+)~\cite{xu2018fairgan,xu2019fairgan+} and Fairness GAN~\cite{sattigeri2019fairness} focus on generating fair datasets composed of input and label pairs, enabling downstream classifiers to achieve fairness concerning sensitive attributes. Alternatively, DB-VAE \cite{amini2019uncovering} employs a de-biased classifier and adaptive resampling of rarer data during training to mitigate the potential biases related to downstream tasks rather than the generative task itself.
Other lines of studies concentrate on the fairness notion related to imbalanced distributions of sensitive attributes.  \citet{choi2020fair} introduced fair generative modeling via weak supervision, utilizing a small, unbiased reference dataset to address imbalanced distribution biases.  

For diffusion models, \citet{sinha2021d2c} addresses bias in datasets generated by unconditional diffusion through a few-shot conditional generation approach. Their goal is to achieve a balanced number of class labels. Similarly, \citet{friedrich2023fair} achieves fairness by fair guidance, which aims to make a balanced number of class labels.

\section{Methodology}


\subsection{Problem Statement}
Our objective is to achieve $\epsilon$-fairness within the generated data by leveraging the independence relationship $\hat{X} \mathrel{\perp \!\!\! \perp} S$~\cite{xu2018fairgan}. 
Moreover, we aim to obtain $p_G$ that ensures not only the distributional independence of sensitive attributes 
$p_G(\hat{X}|S=0) = p_G(\hat{X}|S=1)$,  but also maintains utility with respect to the inherent data manifold. 
While the fairness comes from the independence of generated samples concerning the sensitive groups, the utility of $p_G$ implies that the generated samples from $p_G$ preserve the underlying manifold of original data distribution.

In this section, we demonstrate how our method can achieve these fair characteristics of $p_G$  by introducing a new mechanism to control the sensitive attributes during the sampling. For simplicity, we will omit the subscript $G$ when referring to the distribution associated with the diffusion process and use the notation $ p(X|s) := p(X|S=s)$ when referring to the conditional distribution to attribute $s$. Note that we present the proposed method in a binary context; however, its applicability is broader, covering multi-nominal sensitive attributes or specified text-based conditions.

\subsection{Sampling with Attribute Switching Mechanism}

To make the two data distributions similar, it is well known that controlling high-level features (e.g., semantics, outlines, and image color tones) are important \cite{long2015learning, ilyas2019adversarial}. Recent findings indicate that diffusion models learn distinct attributes at each sampling step \cite{choi2022perception}. Within the context of signal-to-noise ratio (SNR) for diffusion \cite{kingma2021variational}, the sampling process involves three steps: (i) learning \textbf{coarse} features, (ii) generating rich \textbf{content}, and (iii) \textbf{cleaning} up residual noise, as depicted in \Figref{fig:snr}. Specifically, the coarse stage generates high-level features, while the content and clean stages handle fine detail generation \cite{kwon2023diffusion}. 

Thus, through the transfer of high-level features during the sampling process, we might mitigate the inherent distributional disparity between the two sensitive attributes. This approach ensures that a sensitive attribute encompasses coarse features similar to those of the other attribute. To achieve this goal, we propose a general sampling framework to handle sensitive attributes, which we call \textit{attribute switching}.
The sampling process of the proposed method is graphically depicted in Figure \ref{fig:proposed_method_figure}. Simply, at the transition point $\tau$, we alter the condition of the diffusion model, transitioning it from the \textbf{initial sensitive attribute} $s_0$ to the \textbf{switched sensitive attribute} $s_1$.  By leveraging high-level features learned from the distribution of $s_0$, our method can generate images from $p(\hat{X}_0|s_1)$, which exhibits an independent relationship to the sensitive attribute. 
The detailed sampling procedure is explained in Algorithm \ref{alg:sampling}.

\paragraph{Preserving the distribution of the ODE solution} 
We first argue that attribute switching still leads to generate the synthetic data samples on the same data manifold because $p(X_\tau|s_0)$ and $p(X_\tau|s_1)$ have the smooth noised distribution \cite{choi2022perception}.
In a more formal expression, the attribute switching is represented as:
\begin{align}
\label{eq:s_t}
\begin{split}
    S_t &\sim p(S_t=s)\\
    &= p_D(S=s_0)\cdot\delta(s-s_0)\cdot \mathds{1}(t>\tau)\\
    &\quad+p_D(S=s_1)\cdot\delta(s-s_1)\cdot \mathds{1}(t\leq\tau),
\end{split}
\end{align}
for $s_0, s_1 \in \{0,1\}$, $s_0\neq s_1$, and $\mathds{1}$ denotes an indicator function. The vanilla sampling methods, such as DDPM or DDIM, perform sampling under $p(S_t=s)=p_D(S=s)$.  
For mathematical proof, we now demonstrate that the solution to the original SDE maintains an identical distribution to the SDE after attribute switching as follows:
\begin{theorem}
\label{thm:thm2}
Assuming a pre-trained model is trained with \Eqref{eq:forward}, and the subsequent reverse ODE represents an equivalent probability flow ODE corresponding to the pre-trained model \Eqref{eq:reverse_ode}. Then, the solution of following reverse ODE has the same distribution with the pre-trained ODE,
\begin{equation}
    \label{eq:switching_ode}
    d\hat{X_t} = (f(\hat{X_t}, t) - \frac{1}{2}g^2(t)\nabla_x\log p_t(\hat{X_t}|S_t))dt,
\end{equation}
where $S_t $ follows \Eqref{eq:s_t}.
\end{theorem}

\begin{proof}
\textit{(Sketch of the proof)} Using the Fokker-Plank equation, we can show the solution of the given ODE has the same probabilistic as the original ODE. Refer to the Appendix for the details.
\end{proof}

Therefore, we do not need extra training and no adjustments are requisite for existing training methodologies. 
In the following section, we will elucidate how switching can obtain fairness and suggest a straightforward algorithm to find proper $\tau$.


\begin{algorithm}[t]\DontPrintSemicolon
\SetAlgoLined
\SetNoFillComment
    \caption{Fair Sampling with Attribute Switching}
    \label{alg:sampling}
    \KwIn {Pre-trained model $\theta$, subgroup of time steps for sampling  $\{0,k,\ldots,T-k,T\}$, transition point $\tau$, and sensitive attributes $s_0, s_1$}
    \KwOut {Generated fair images $\hat{X}_0$}
    \For {a random batch}{
    \textbf{Initialize:} $\hat{X}_T\sim N(0, I)$

    \For {$t=T-k, T-2k,\ldots,\tau-k$ }{
    Sampling under the condition of $S=s_0$
    
    $\quad d\hat{X_t} = (f(\hat{X_t}, t) - \frac{1}{2}g^2(t)\psi_{\theta}(\hat{X}_t, t, s_0)dt
    $
    }
    \For {$t=\tau, \cdots, 0$ }{
    \textit{Switching}: sampling under condition of $S=s_1$ 
    
    $\quad
    d\hat{X_t} = (f(\hat{X_t}, t) - \frac{1}{2}g^2(t)\psi_{\theta}(\hat{X}_t, t, s_1)dt
    $
    }
    }
\end{algorithm}

\begin{algorithm}[t]\DontPrintSemicolon
\SetAlgoLined
\SetNoFillComment
    \caption{Fair Transition Point Searching}
    \label{alg:search_tau}
    \KwIn {Pre-trained model $\theta$, sensitive attribute $s_0, s_1$, subgroup of time steps for sampling  $\{0,k,\ldots,T-k,T\}$}
    \KwOut {Transition point $\tau$}
    \textbf{Initialize:} $X_T\sim N(0, I)$
    
    \For {$t=T-k, T-2k,\ldots,0$ }{   
        \tcc{\emph{Single sampling for $\tau$ search}}
        
        Calculate score $\psi_{\theta}(X_t, t,s_0)$, $\psi_{\theta}(X_t, t,s_1)$ using pre-trained model $\theta$

        Obtain diffusion coefficient $g(t)$ in \Eqref{eq:forward}
        
        Store difference between $s_0$ and $s_1$ using $\theta$ \\ $\qquad D_t = g^2(t)\{\psi_{\theta}(X_t, t,s_0)-\psi_{\theta}(X_t, t,s_1)\}$ } 

    Calculate the transition point $\tau$ 
    $$\tau = \arg\min_\tau \left\|\sum_{i\leq\tau}  D_i - \\ \sum_{i>\tau} D_i\right\|$$
\end{algorithm}


\subsection{Fair Sampling with Attribute Switching}

As vanilla diffusion sampling fails to satisfy $\epsilon$-fairness, two data distributions conditioned on the sensitive attribute become distinct as denoising progresses from time $T$ to $0$. This is illustrated in \Figref{fig:proposed_method_figure}, where `Path0' and `Path1' diverge from each other. In contrast, our method facilitates attaining the distribution matching for $p(\hat{X}_0|s_0)$ and $p(\hat{X}_0|s_1)$ by switching the condition of the sensitive attribute of \Eqref{eq:reverse_ode} from $s_0$ to $s_1$ at the transition point $\tau$. For a detailed explanation, we now present a theoretical analysis aimed at identifying an optimal transition point $\tau$ that guarantees fairness. We also provide theoretical support for the practicality of utilizing a pre-trained diffusion model.

\paragraph{Fair Condition of Transition Point $\boldsymbol{\tau}$}
Our primary goal is to achieve a distribution of generated samples wherein $p(\hat{X}|S = s_0) = p(\hat{X}|S = s_1)$. Thus, we suggest the condition of transition point $\tau$ to achieve the goal as follows:
\begin{theorem}\label{thm:thm_fair}
    \textbf{(Fair condition of transition point $\boldsymbol{\tau}$)} 
    Let $\tau$ be a transition point satisfying the following condition:
    \begin{equation}
    \label{eq:tau_sum}
    \int_0^\tau D(t)dt  = \int_\tau^T D(t)dt.
    \end{equation}
    where
    \begin{equation}
    D(t) = g^2(t)(\nabla_x\log p_t(\bar{X_t}|s_0) -\nabla_x\log p_t(\bar{X_t}|s_1)),
    \label{eq:diff}
    \end{equation}
    Then, 
    the generated distribution from the following reverse ODE
    becomes independent of the sensitive attribute, where $S_t$ is from \Eqref{eq:s_t}.
    
\end{theorem}
\begin{proof} \textit{(Sketch of the proof)} By using \Eqref{eq:reverse_ode}, we can get the equation about the generated image from reverse ODE. Moreover, by \Eqref{eq:s_t} and (\ref{eq:switching_ode}), we can get the equation about the switched image. Refer to the Appendix for the details.
\end{proof}

The above theorem states that the attribute switching guarantees fairness of the generated data distribution, by ensuring the independence of the sensitive attribute when $\tau$ satisfies \Eqref{eq:tau_sum}.


\begin{figure}
    \centering
    \includegraphics[width=\linewidth]{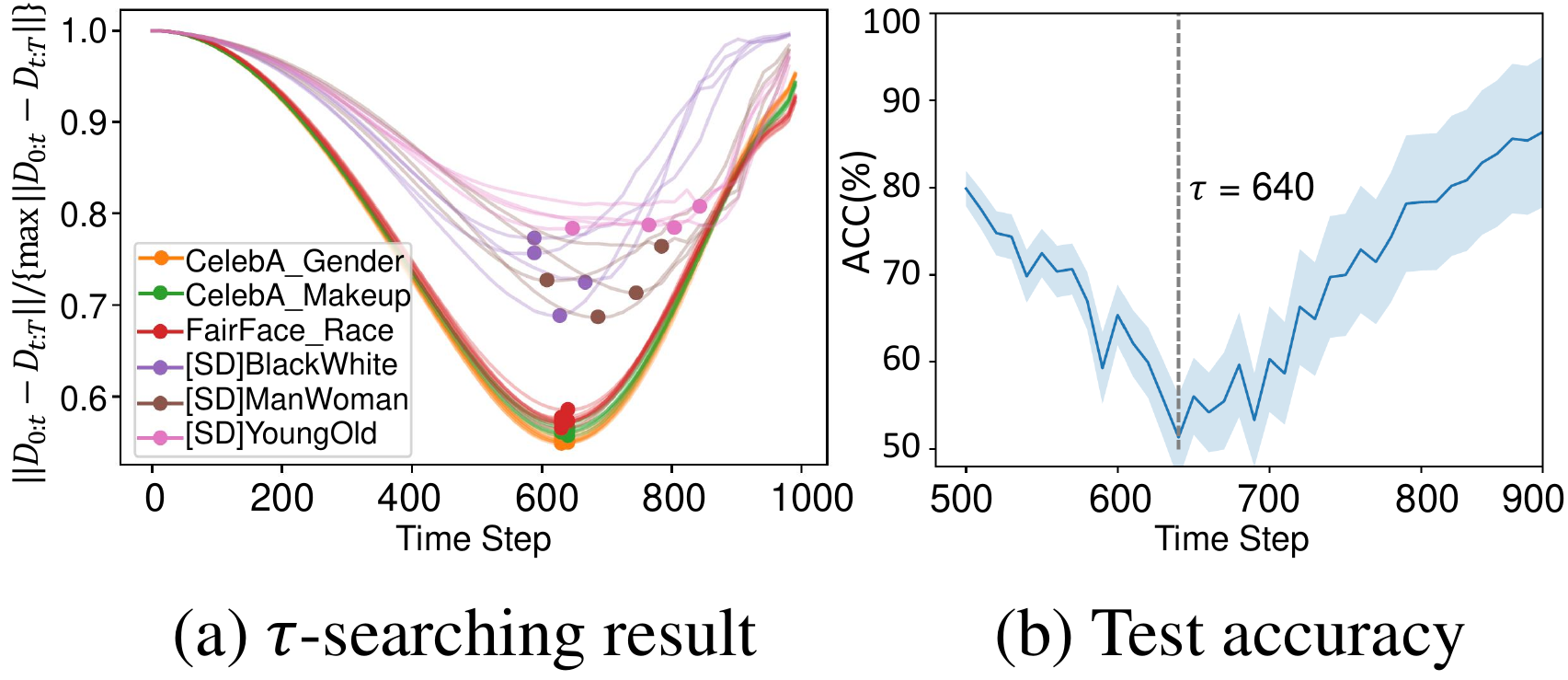}
    \caption {$\tau$-search algorithm and its efficiency. (a) $\tau$-searching results with diverse diffusion models and datasets and (b) test accuracy of true data evaluated with models trained using synthetic data with varying $\tau$ values.}
    \label{fig:u_shape_total}
\end{figure}

\paragraph{Empirical Validation of Transition Points $\boldsymbol{\tau}$}
To guarantee fairness in diffusion sampling according to \Thmref{thm:thm_fair}, our goal is to find $\tau$ that minimizes the following objective:
{\footnotesize\begin{equation}
    \label{eq:tau_diff}
    \left\|\sum_{i\leq\tau}  D_i - \\ \sum_{i>\tau} D_i\right\|.
\end{equation}}
Thus, we first introduce an efficient algorithm to determine an appropriate $\tau$ in Algorithm \ref{alg:search_tau}. Within a mini-batch setting, we store the estimated value of the difference between two sensitive attributes in \Eqref{eq:diff} and estimate the score in \Eqref{eq:loss_scorematching}. The optimal $\tau$ that minimizes \Eqref{eq:tau_diff} can be found by performing a one-time sampling with cumulative sum calculation for a given mini-batch. 
To verify whether the selection of $\tau$ from a single mini-batch can be applied to the entire sampling process, we illustrated the maximum scaled values of \Eqref{eq:tau_diff}, varying the model and dataset in \Figref{fig:u_shape_total}(a).
The top three correspond to experiments conducted with a batch size of 256, while the remaining three at the bottom represent experiments with a batch size of 5.
For a large batch size, individual computations of differences for each mini-batch are unnecessary; a single batch calculation suffices for determining $\tau$ for a given dataset. More detailed experiments with diverse batch sizes are illustrated in the Appendix.
Note that the function $g(t)$ in \Eqref{eq:diff} might be differently defined with the scheduler and sampler \cite{ho2020denoising,karras2022elucidating}.

To verify the optimal transition point $\tau$ from Algorithm \ref{alg:search_tau}, we plot the accuracy of a model trained with the synthetic data and test with the original FairFace \cite{karkkainen2021fairface} ``race" (black-white) dataset. 
In Figure \ref{fig:u_shape_total}(b), we observe that the theoretically derived value of $\tau = 640$, which ensures fairness, closely aligns with the $\tau$ value that exhibits fairness in practice. Interestingly, this value is highly aligned to their empirical boundaries of coarse-content phases ($\approx \text{SNR}=10^{-2}$)  \cite{choi2022perception} in \Figref{fig:snr}.
These empirical results highly support transferring high-level features of $s_0$ to content phases of $s_1$ is effective for fair sampling.

\begin{table}[!ht]
\centering
\resizebox{0.88\linewidth}{!}{%
\begin{tabular}{l|l|rrrr}
\hline
 Classifier & Methods & \multicolumn{1}{l}{$S = 0$} & \multicolumn{1}{c}{$S = 1$} & \multicolumn{1}{c}{gap} & \multicolumn{1}{c}{BER} \\ \hline
\multicolumn{1}{c|}{} & Real & 9.51 & 4.89 & 4.62 & 7.20 \\ \hline
\multirow{4}{*}{\begin{tabular}[c]{@{}c@{}}Syn (Tr)\\$\rightarrow$\\ Orig (Te)\end{tabular}} & Vanilla & 8.80 & 8.20 & 0.60 & 8.50 \\
 & Mixing & 39.59 & 43.93 & 4.34 & 41.76 \\
 & Editing & 9.58 & 18.03 & 8.46 & 14.42 \\ 
 & Ours & 54.63 & 54.92 & 0.29 & \textbf{54.78} \\ \hline
\multirow{4}{*}{\begin{tabular}[c]{@{}c@{}}Orig (Tr)\\$\rightarrow$\\ Syn (Te)\end{tabular}} & Vanilla & 19.59 & 10.68 & 8.91 & 15.14 \\
 & Mixing & 62.64 & 20.55 & 42.09 & 41.60 \\
  & Editing & 31.56 & 10.92 & 20.64 & 21.24 \\
 & Ours & 62.59 & 38.86 & 23.73 & \textbf{50.73} \\ \hline
\end{tabular}
}
\caption{Error rates for classifier (\%), trained (Tr) with synthetic data (Syn) generated from each sampling method and tested (Te) with original data (Orig), and vice versa.}
\label{tab:epsilon_fairness}
\end{table}

For complexity analysis, let $N_B = \lceil N/\text{batch\_size} \rceil$ be the number of batch samplings for $N$ data samples. The attribute switching sampling in \Algref{alg:sampling} requires $O(N_B \times T)$ for $T$ denoising steps, the same as vanilla sampling since switching only requires \texttt{if-else} without additional space complexity.
For $\tau$ searching in \Algref{alg:search_tau}, we only need one-time mini-batch sampling of $O(T)$, which is marginal to fair sampling of $O(N_B \times T)$. Only for \Algref{alg:search_tau}, we require twice the memory space for sampling both sensitive attributes and comparing their step-wise differences. 

\begin{figure*}[!t]
\centering     
    \subfigure[Bird-Truck]{\label{fig:bird_truck_pca}\includegraphics[width=45mm]{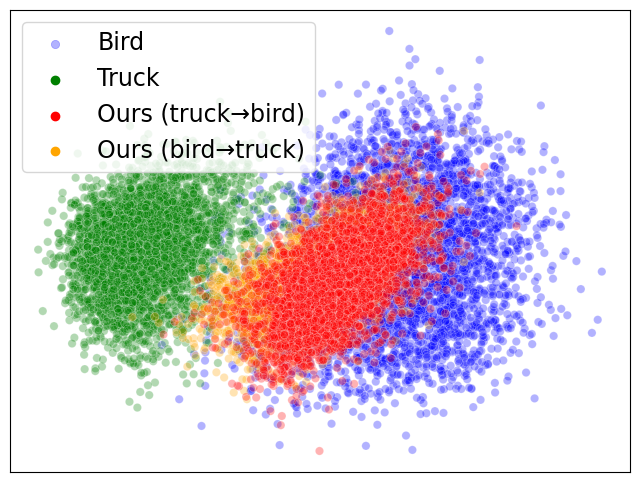}}
    \subfigure[Bird-Frog]{\label{fig:bird_frog_pca}\includegraphics[width=45mm]{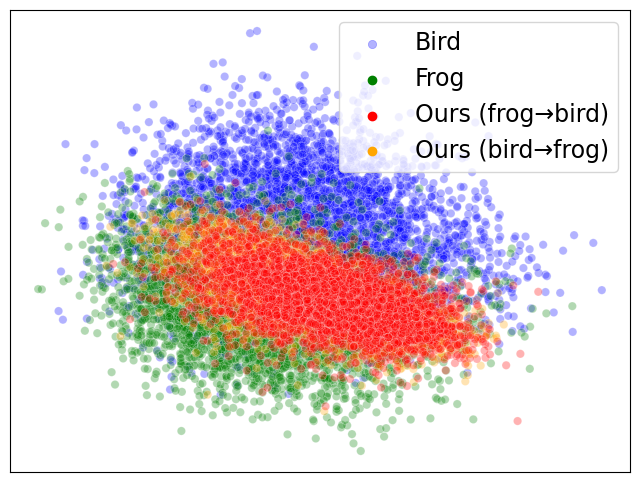}}
    \subfigure[Frog-Truck]{\label{fig:frog_truck_pca}\includegraphics[width=45mm]{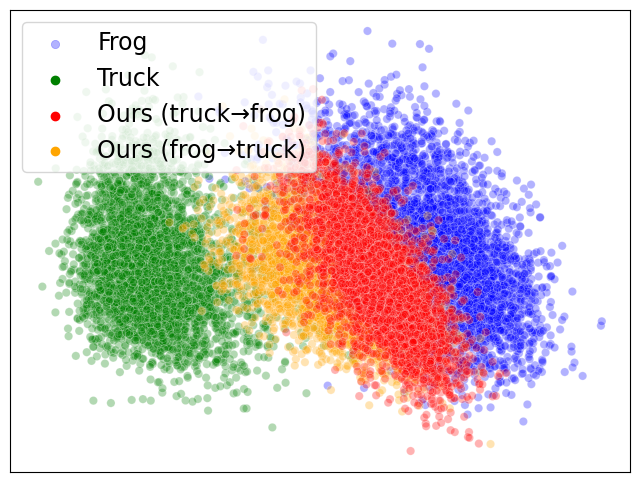}}
    \caption{Dimension reduction (PCA, $n=2$) results of the data distribution for both the true data and the generated data of the proposed methods for Bird-Truck, Bird-Frog, and Frog-Truck pairs within the CIFAR-10 datasets.}
    \label{fig:pca}
\end{figure*}

\begin{table*}[!htb]
\centering
\resizebox{\textwidth}{!}{%
\begin{tabular}{c|l|rrrr|rrrr|rrrr}
\hline
\multirow{2}{*}{\begin{tabular}[c]{@{}c@{}} \end{tabular}}  & \multirow{2}{*}{\begin{tabular}[c]{@{}c@{}}Methods\end{tabular}} & \multicolumn{4}{c|}{FairFace Race (12k)} & \multicolumn{4}{c|}{CelebA Male (12k)} & \multicolumn{4}{c}{CelebA Heavy-Makeup (12k)} \\ \cline{3-14} 
 &  & \multicolumn{1}{l}{$S =$ 0} & \multicolumn{1}{l}{$S =$ 1} & \multicolumn{1}{l}{total} & \multicolumn{1}{l|}{gap} & \multicolumn{1}{l}{$S =$ 0} & \multicolumn{1}{l}{$S =$ 1} & \multicolumn{1}{l}{total} & \multicolumn{1}{l|}{gap} & \multicolumn{1}{l}{$S =$ 0} & \multicolumn{1}{l}{$S =$ 1} & \multicolumn{1}{l}{total} & \multicolumn{1}{l}{gap} \\ \hline
\multicolumn{1}{c|}{\multirow{3}{*}{Total}} & Real (Test) & 10.12* & 8.77* & 6.35* & 1.34* & 21.66 & 21.25 & 2.50 & 0.41 & 12.92 & 18.88 & 2.43 & 5.96 \\
\multicolumn{1}{c|}{} & Vanilla (Total) & 29.60 & 34.99 & 29.50 & 5.40 & 22.24 & 19.75 & 7.19 & 2.49 & 19.75 & 22.24 & 7.19 & 2.49 \\
\multicolumn{1}{c|}{} & Ours (Total) & 31.15 & 36.23 & 31.11 & 5.08 & 23.17 & 29.20 & 7.83 & 6.03 & 20.35 & 21.36 & 7.59 & 1.01 \\ \hline
\multicolumn{1}{c|}{\multirow{3}{*}{\begin{tabular}[c]{@{}c@{}} $s_0=0$\\$s_1=1$\end{tabular}}} & Real ($s_0$)& 21.47* & 33.81* & 17.43* & 12.33* & 2.36 & 68.11 & 22.81 & 65.75 & 2.64 & 53.13 & 13.78 & 50.49 \\
\multicolumn{1}{c|}{} & Vanilla ($s_0$) & 28.04 & 40.59 & 31.78 & 12.55 & 7.04 & 66.00 & 23.03 & 58.96 & 8.75 & 47.37 & 14.50 & 38.62 \\
\multicolumn{1}{c|}{} & Ours ($s_0\rightarrow s_1$) & 31.05 & 38.29 & 32.15 & 7.25 & 8.50 & 58.54 & 18.54 & 50.05 & 10.87 & 39.80 & 11.72 & 28.92 \\ \hline
\multicolumn{1}{c|}{\multirow{3}{*}{\begin{tabular}[c]{@{}c@{}} $s_0=1$\\$s_1=0$\end{tabular}}} & Real ($s_0$) & 29.68* & 18.19* & 13.05* & 11.49* & 53.25 & 2.12 & 20.02 & 51.13 & 53.25 & 2.12 & 20.02 & 51.13 \\
\multicolumn{1}{c|}{} & Vanilla  ($s_0$) & 33.77 & 32.53 & 30.59 & 1.24 & 58.31 & 9.14 & 19.12 & 49.17 & 49.38 & 6.66 & 19.03 & 42.72 \\
\multicolumn{1}{c|}{} & Ours ($s_0\rightarrow s_1$) & 32.08 & 35.57 & 31.26 & 3.48 & 50.03 & 12.00 & 15.64 & 38.03 & 42.03 & 8.44 & 15.44 & 33.60 \\ \hline
\end{tabular}
}
\caption{FID comparison of generated images with a fixed start point for FairFace and CelebA.}
\label{tab:fid_total}
\end{table*}

\section{Experiments}
\subsection{Experimental Setups}
\label{sec:exp_setting}
The success of diffusion models is attributed to their ability to provide effective guidance in a manifold, performing well even under linear interpolation \cite{ramesh2022hierarchical}. As our approach represents the first attempt to investigate fairness by matching the distribution given sensitive attributes in diffusion models, the most straightforward method for comparison is mixing conditional embeddings \cite{song2020denoising}. Specifically, we linearly combine the sensitive attribute embeddings for sampling, utilizing the embedding of $s_1$ with a probability $p>0.5$, and that of $s_0$ with a probability of $1-p$.
As another baseline, we employ DiffEdit \cite{couairon2022diffedit}, a state-of-the-art image diffusion-based editing method. 
We generate images using the original sampling method  of diffusion models and edit them by subtracting corresponding text for $s_0$ and adding corresponding text for $s_1$. 



Following \citet{xu2018fairgan}, we employ a classifier trained on synthetic data and evaluate it on the original dataset to measure fairness. Specifically, we train a ResNet18 \cite{he2016deep} model as a classifier to predict the sensitive attribute $S$, achieving a training accuracy of 100\% in most cases.
For datasets, we employ FairFace, CelebA \cite{liu2018large} to quantify data fairness and utility, and CIFAR-10\cite{krizhevsky2009learning} for dimension reduction experiments. 
Regarding evaluation metrics, we assess classification performance using accuracy, data $\epsilon$-fairness with BER as defined in \Eqref{eq:ber}, and data utility with the Frechet Inception Distance (FID) score.
All experiments are conducted using PyTorch-based diffusion libraries \cite{von-platen-etal-2022-diffusers}. Additional details regarding the experimental setup can be found in the Appendix and the code is available at \url{https://github.com/uzn36/AttributeSwitching}.



\subsection{Data Fairness}
Table \ref{tab:epsilon_fairness} presents the evaluation of error rates for the original and synthetic data, generated using the same diffusion model trained on the FairFace dataset, with different sampling methods: vanilla, mixing, editing, and attribute switching.
The first column indicates the training and evaluating data, e.g., ``Orig (Tr) $\rightarrow$ Syn (Te)" represents the classifier trained on the original dataset and evaluated on synthetic data. ``Real" indicates training and testing with the original dataset. The values are presented as error rates in percentage, where the column for $S= s$ means $P(f(X)=1-s|S = s)\cdot$ 100 (\%) for the classifier $f$, and BER is obtained by \Eqref{eq:ber}.
Notably, we achieved fairness in terms of both the gap and BER, for both ``Orig (Tr) $\rightarrow$ Syn (Te)" and ``Syn (Tr) $\rightarrow$ Orig (Te)" evaluations. Our approach achieved a BER similar to 50\% while mixing and editing only achieved 42\%, 14\%, respectively. Moreover, the classifier trained with the original dataset tends to classify the generated data more often as $S = 1$, while this tendency is alleviated in attribute switching.
This conclusion demonstrates that attribute switching has successfully achieved fairness in terms of $\epsilon$-fairness.\\
Another method to evaluate the similarity of the generated data distribution is through dimension reduction and visualization.
To confirm that the proposed methods generate fair data, we performed dimension reduction (PCA) to dimension $n=2$ on the pre-trained ResNet embedding space. To ensure distinct distributions for each sensitive attribute, we extracted three different classes from the CIFAR-10 dataset: bird, truck, and frog. Each class is considered a sensitive attribute, and then, we applied our method to generate distributionally equivalent images for all possible pairs. For pairwise distribution analysis, we calculated principal component vectors using true data distribution pairs. Subsequently, we projected the generated data onto these principal dimensions. The results are depicted in \Figref{fig:pca}. The generated data is placed in between the two true distributions while sharing similar distributions between attributes.

\subsection{Data utility}  
To assess data utility, we compute the FID using the training dataset and an equal number of generated samples. 
 It's worth noting that the FID score is highly sensitive to the dataset size, so we aimed to maintain a consistent number of samples. However, due to limitations in the FairFace dataset, we encountered insufficient data when calculating the FID with the real test data. In cases where the dataset sizes did not align, we have marked this discrepancy with an asterisk (*).

As our method switches the attribute from $s_0$ to $s_1$ and labels the sensitive attribute as $s_1$, we conducted a comparison of FID with a fixed start point. The results are presented in Table \ref{tab:fid_total}. We measure the FID with a total of 12k train data, 12k train data with $S=0$, and 12k train data with $S=1$.
As shown in the table, the total FID of attribute switching is similar to that of vanilla sampling. This suggests that the total data distribution of generated data is similar to that of vanilla data, which can ensure the data utility. 
Notably, the FID of the proposed method is worse than vanilla sampling, in terms of $s_0$. This is expected since the attribute of generated data is changed to $s_1$. Most importantly, the FID with $s_1$ is significantly better even when the gap between the two attributes is reduced.
Given that FID measures the difference between two distributions, we can conclude that this result pertains not only to data utility but also to fairness considerations. This finding further supports the notion that attribute switching contributes to making the distribution of generated images more similar.

\begin{figure}[!t]
\centering     
     \subfigure
    {\label{fig:man2womanm}\includegraphics[width=20mm]{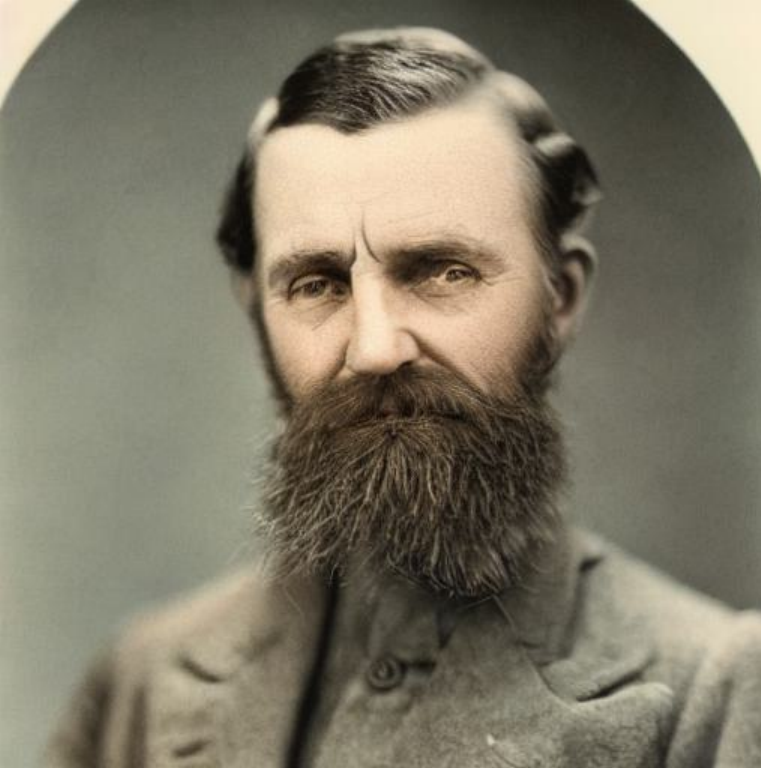}}  \subfigure{\label{fig:woman2manw}\includegraphics[width=20mm]{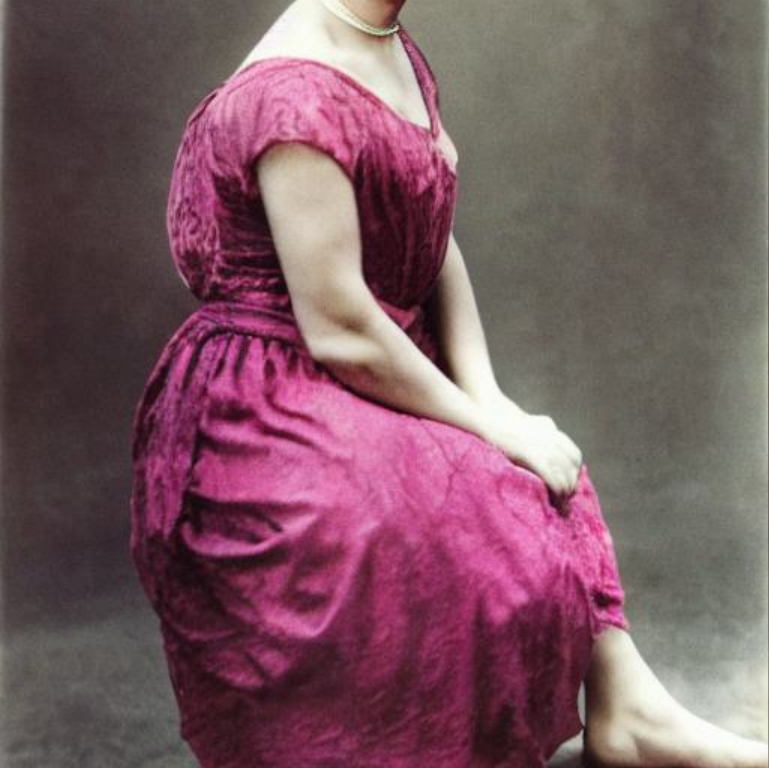}}
    {\label{fig:young2oldy}\includegraphics[width=20mm]{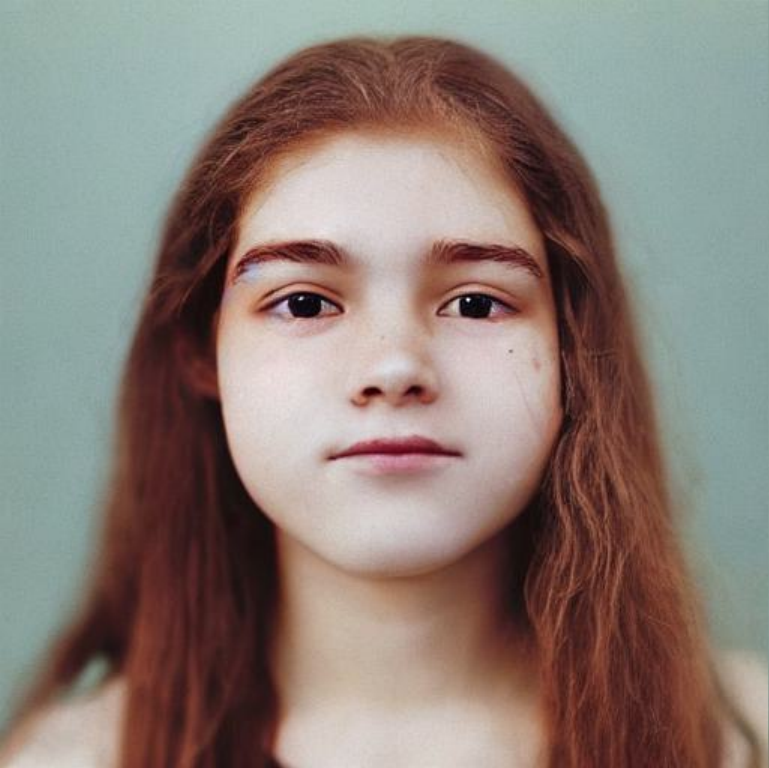}}
    \subfigure
    {\label{fig:old2youngo}\includegraphics[width=20mm]{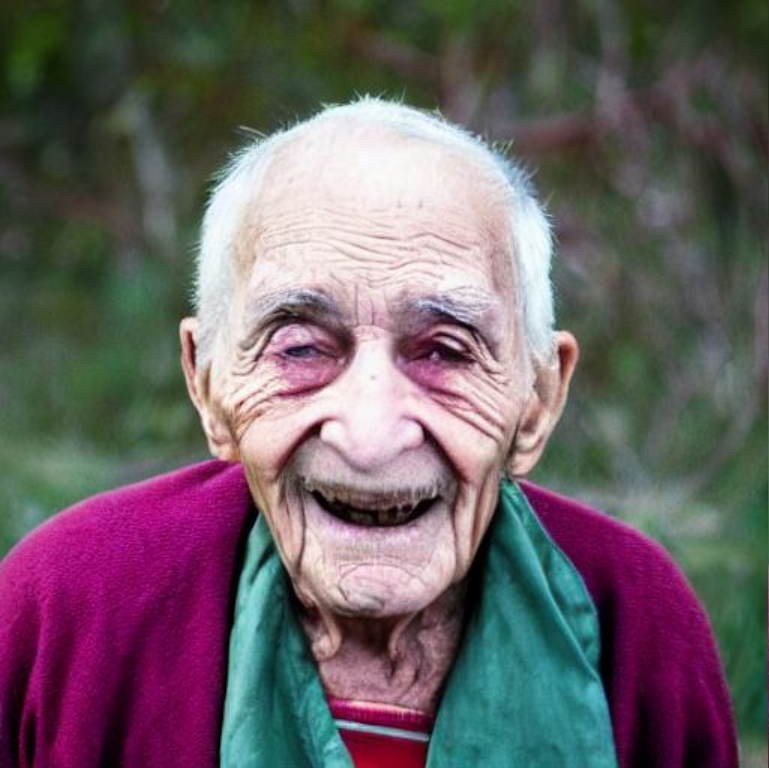}}\\
    \subfigure
    {\label{fig:man2womanw}\includegraphics[width=20mm]{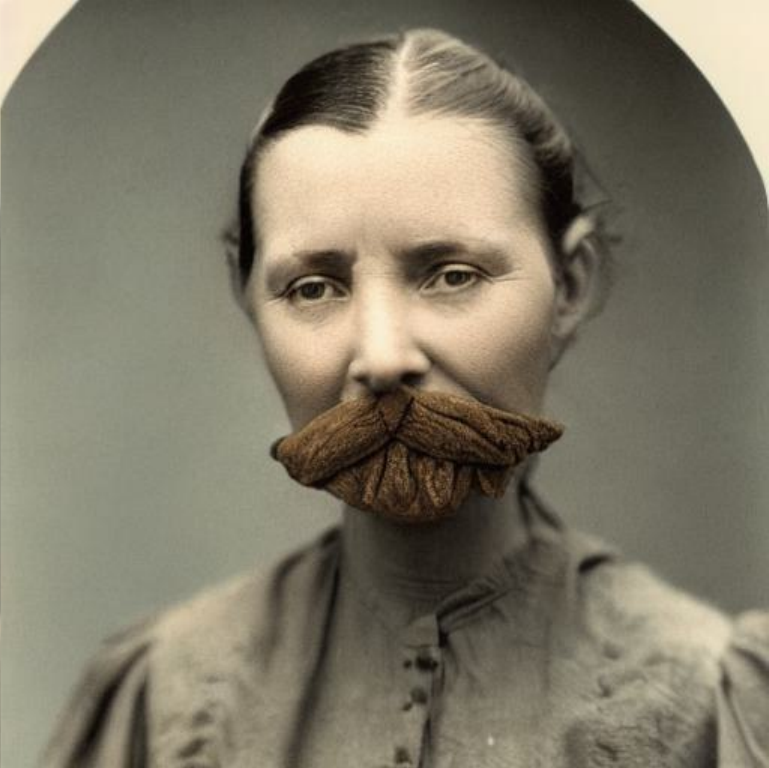}}  
    \subfigure
    {\label{fig:woman2manm}\includegraphics[width=20mm]{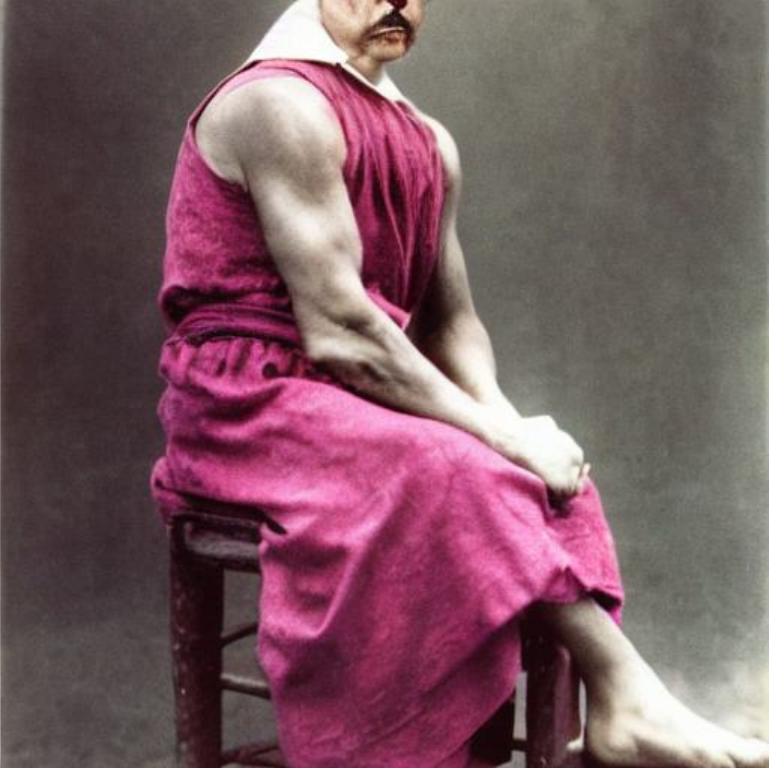}}
       \subfigure
    {\label{fig:young2oldo}\includegraphics[width=20mm]{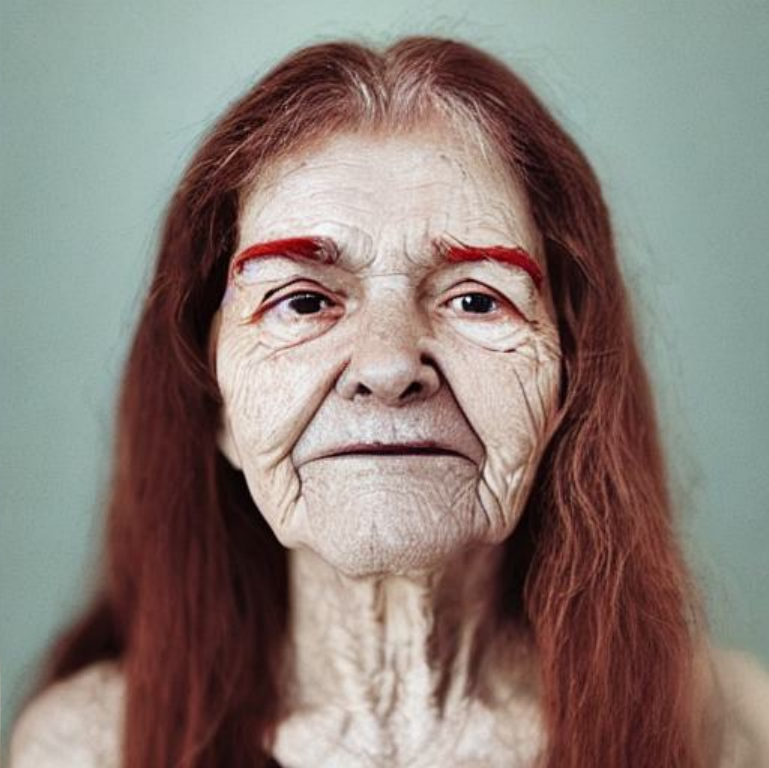}}
    \subfigure
    {\label{fig:old2youngy}\includegraphics[width=20mm]{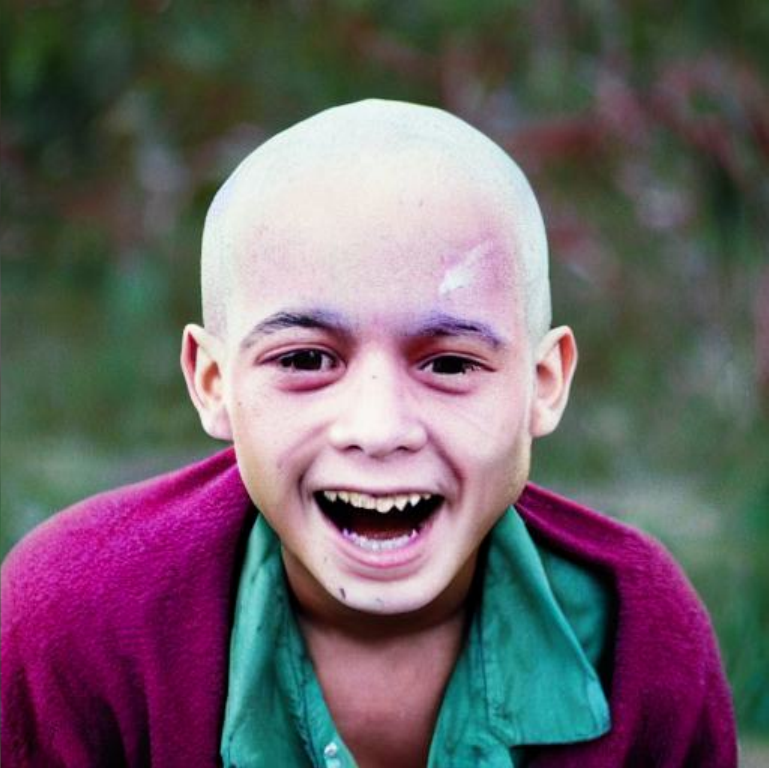}}
    \caption{Sampling from stable diffusion with (Top) vanilla sampling and (Bottom) switching sensitive attribute by the proposed method. After switching, the uncontrolled features such as hairstyle, dress, and mustache are maintained.}
    \label{fig:swt_stable_diff}
\end{figure}
Attribute switching can also ensure data utility for downstream tasks, particularly in training classifiers  where sensitive attributes, such as ``Male," are removed. For example, consider the ``Smiling" classification task with the CelebA dataset. We conditioned a diffusion model on both the class ``Smiling" and the sensitive attribute ``Male". The classification results are in \Tabref{tab:smilemale}.
\begin{table}[!t]
\centering
\begin{tabular}{l|cc||c|cc}
\hline
{Acc (\%) }& Smile ($\uparrow$) & Male (-)  & {}& Smile & Male \\ \hline
\multicolumn{1}{c|}{Vanilla} & \multicolumn{1}{r}{88.38} & \multicolumn{1}{r||}{93.04} & Ours & \multicolumn{1}{r}{88.14} & \multicolumn{1}{r}{55.71} \\ \hline
\end{tabular}
\caption{Classification accuracy trained with synthetic data.}
\label{tab:smilemale}
\end{table}
The classifier trained with synthetic data from vanilla sampling achieved high test accuracy in both ``Smiling" and ``Male", indicating both attributes are predictable by the synthetic data. In contrast, attribute switching maintains ``Smiling" accuracy while achieving fair ``Male" accuracy, close to 50\%.
These results indicate that attribute switching only modifies the sensitive attribute, without significantly affecting other attributes. This implies that attribute switching can be employed across various downstream tasks without raising fairness concerns.



\subsection{Sampling with Text-conditioning Models}
As we mathematically prove in \Thmref{thm:thm2}, we now demonstrate that the proposed method can be easily applied to any pre-trained model, regardless of whether the model is conditioned with multi-class labels or text information, just as \cite{balaji2022ediffi}. 
We used widely used pre-trained diffusion models with text conditional embedding, i.e., stable diffusion models \cite{rombach2022high}. For the sensitive attributes $s_0$ and $s_1$, we use guided text to generate images as follows: from \textit{``a color photo of the face of $\underline{\quad s_0 \quad}$"} to \textit{``a color photo of the face of $\underline{\quad s_1 \quad}$"}.
\Figref{fig:swt_stable_diff} illustrates the generated images using text conditioning: \{\textit{``man"},\textit{``woman"}\} and \{\textit{``young person"},\textit{``old person"}\} with $\tau$ determined by averaging four results using a batch size of 5, as in \Figref{fig:u_shape_total}(a). In \Figref{fig:swt_stable_diff}, 
each image sampled with a switching mechanism has characteristics that were not present with the vanilla sampling generated (e.g., a woman with a mustache and a man with a pink dress).
More generated images of Figures with different text guidance and results of mixing embedding can be found in the Appendix.

\begin{figure}[!t]
\centering     
    \includegraphics[width=80mm]{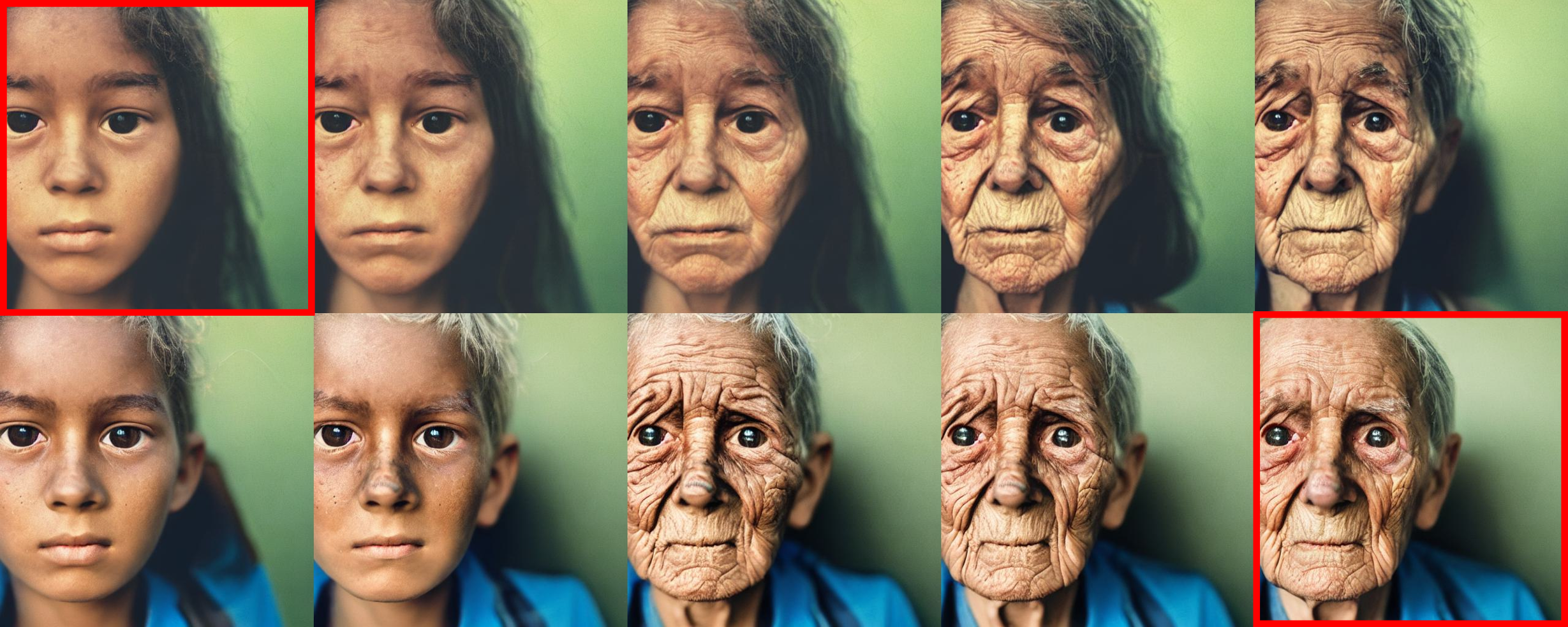}
    \includegraphics[width=80mm]{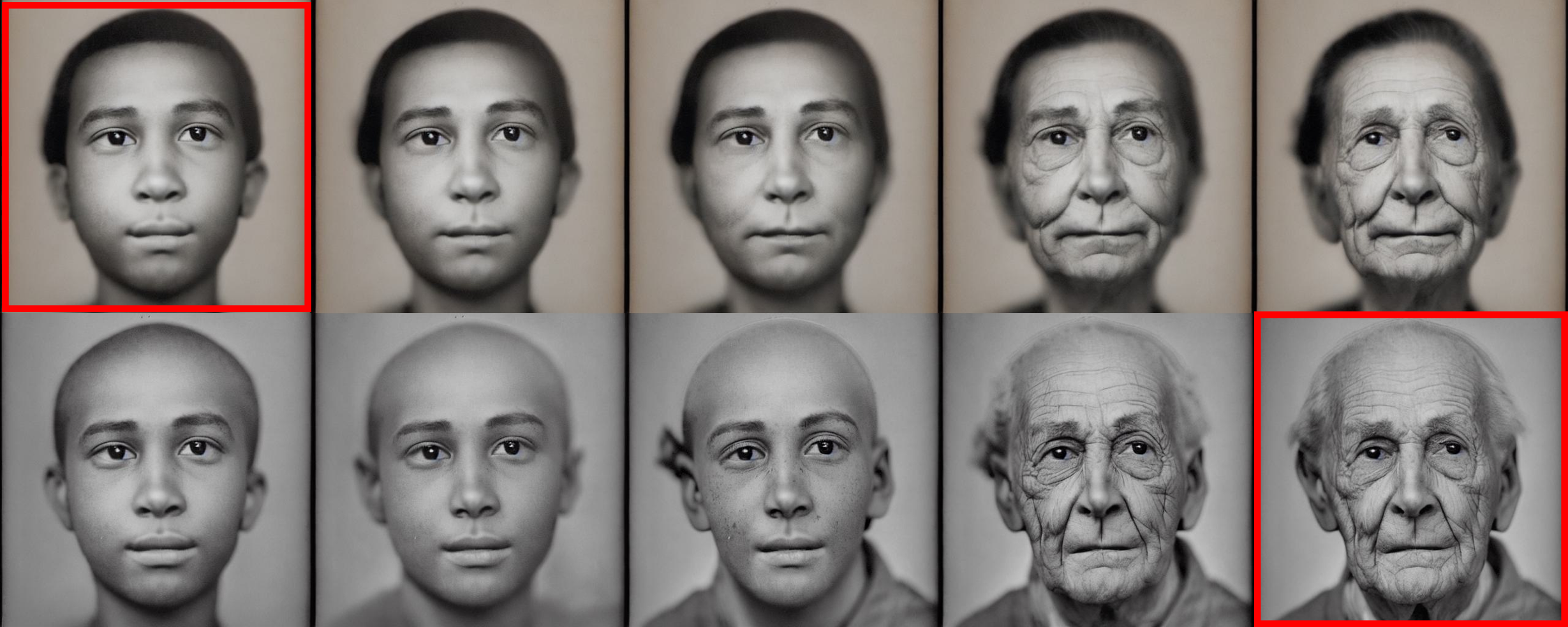}
    \caption{Generated pairwise images with the same random seeds. In each row, images with red borderlines denote vanilla sampling images. Beginning with these vanilla images, we progressively increase the value of $\tau$, moving further away from the initial vanilla image.}
    \label{fig:figure_stablediffusion_fixseed}
\end{figure}


\Figref{fig:figure_stablediffusion_fixseed} shows pair-wise images with same random seed but altering the $s_0\rightarrow s_1$ and  $s_1\rightarrow s_0$ during sampling. We use the text condition of \{\textit{``young person"},\textit{``old person"}\} and plot $\tau\in\{500,600,700,800,1000\}$. By choosing the initial sensitive condition $s_0$ and switching condition $s_1$, we can control what property would be retained and what property would be modified. For example, the top image pairs show similar ages of intermediate images, while $s_0\rightarrow s_1$ and  $s_1\rightarrow s_0$ keep female and male properties, respectively.

\section{Conclusion \& Limitations}
In this paper, we investigated the fair sampling in diffusion models, focusing on $\epsilon$-fairness to generate distributionally similar images. To address the limitations of existing fairness frameworks, we present a novel switching sampling method that generates data satisfying both fairness and utility. We theoretically prove and experimentally show the effectiveness of the proposed methods in various settings, including pre-trained diffusion models with text embeddings.

While attribute switching preserves high-level features of images, it allows to generate the images with similar distributions. Therefore, it can generate images that were previously challenging, such as bald women.
However, elements with significant contextual impact tend to be removed during the process, which should be addressed later carefully.
\clearpage
\section{Ethics Statement}
This study focuses on addressing fairness issues in generative models, particularly in the context of image generation using diffusion models. Generative models have a broad spectrum of applications ranging from data augmentation and synthetic data creation for privacy preservation to generating novel content for creative and entertainment purposes. However, there is an inherent risk that these models may perpetuate or even exacerbate fairness issues, such as reinforcing stereotypes or introducing biases in their applications.

Our proposed method aims to enable pre-trained generative models to produce images, leveraging similar distributions across the groups defined by sensitive attributes such as race or gender. This approach avoids the need for additional classifier training, making it computationally efficient and readily applicable to existing models.

However, it's important to note that our study primarily focuses on distributional fairness, one facet of the multifaceted concept of fairness in generative models. Moreover, many existing studies primarily explore fairness within specific downstream tasks, neglecting broader considerations. A comprehensive understanding of fairness in generative models demands the counterfactual investigation of multiple factors but has been somewhat unexplored. 

We hope that our work will highlight the importance of evaluating and addressing fairness issues in generative models more comprehensively. Future research should expand upon this foundation, exploring various fairness factors in depth to ensure more equitable and unbiased applications of generative AI technologies.

\section{Acknowledgements}
This research was supported by the National Research Foundation of Korea (NRF) grant funded by the Korean government (MSIT) (No. RS-2023-00272502, No. NRF-2022R1F1A1065171, No. 2022R1A5A6000840) and  Institute of Information \& communications Technology Planning \& Evaluation (IITP) grant funded by the Korea government (MSIT) (No. 2020-0-01336, Artificial Intelligence graduate school support (UNIST)).
\bibliography{aaai24}
\clearpage
\appendix
\section{Implementation Details}
\paragraph{Diffusion Model}
We utilized the GitHub code available at \url{https://github.com/VSehwag/minimal-diffusion} to train diffusion models on the FairFace dataset \cite{karkkainenfairface} and the CelebA dataset \cite{liu2015faceattributes}.
For the FairFace dataset, we conducted training with a batch size of 128, a learning rate of 5e-05 , with the number of epochs 150. As for the CelebA dataset, all settings remained the same as for the FairFace dataset, except for the number of epochs: for the CelebA dataset, we trained for a total of 100 epochs. 
Specifically, we employed the FairFace dataset with the ``Race" label and selected two out of the seven race labels: ``Black" and ``White". In the case of the CelebA dataset, we used the labels ``Male" and ``Heave-Makeup".
Algorithm \ref{alg:sampling} was implemented with certain modifications to the minimal-diffusion GitHub code.
For the CIFAR10 dataset sampling used in Figure \ref{fig:pca}, we performed sampling from a pre-trained diffusion model form \cite{karras2022elucidating}, a state-of-the-art model. 
For the pre-trained stable diffusion model, we referenced the code provided at \url{https://github.com/huggingface/diffusers}.
\paragraph{Classifier}
To train the classifier for the Table \ref{tab:epsilon_fairness}, we employed a ResNet18~\cite{he2016deep} on the FairFace dataset. We utilized the SGD optimizer with a learning rate of 0.1 and a Cosine scheduler along with a momentum of 0.9, for a total of 100 epochs. For the model trained with the original dataset, we saved the model achieving the highest validation accuracy. This validation set was derived from the original test set. For this model, we utilize the entire training dataset for training. For training with synthetic data, we sampled 6400 instances with label 0 and an equal number of instances with label 1, resulting in a total of 12800 training examples. The model with the highest training accuracy, achieving 100\% accuracy, was saved.
\paragraph{$\tau$-searching}
For the diffusion model, trained on both the CelebA and FairFace datasets, we determined $\tau = 640$ for each attribute (Male and Heavy-Makeup for CelebA, Race for FairFace). This was achieved using a one-step search approach with a batch size of 256 and a search step of 10.
For the stable diffusion model, we utilized a batch size of 5 and a search step of 20. After repeating the search process 4 times and averaging the results, we found $\tau = 700$ for text conditioning on \{\textit{``man"}, \textit{``woman"}\}, and $\tau = 760$ for \{\textit{``young person"}, \textit{``old person"}\}.

\section{Proofs}
In this section, we provide a mathematical proof of \Thmref{thm:thm_fair} and \Thmref{thm:thm2}. 

\begin{T2}
(Restated)
    \textbf{(Fair condition of transition point $\boldsymbol{\tau}$)} 
    Let $\tau$ be a transition point satisfying 
    \Eqref{eq:tau_sum}
    where satisfying \Eqref{eq:diff}.
    Then, 
    the generated distribution from the following reverse-ODE
    becomes independent of the sensitive attribute, where $S_t$ is from \Eqref{eq:s_t}.
\end{T2}

\begin{proof}
Let $X$ be the latent variable of the forward process in \Eqref{eq:forward}, and $\bar{X}$ be the latent variable of the reverse process in \Eqref{eq:reverse_ode}. 
Then, by \Eqref{eq:reverse_ode},
\begin{align*}
    \bar{X_T}^{(s_i)} - \bar{X_0}^{(s_i)} = \bar{X_T}^{(s_i)} - \bar{X_\tau}^{(s_i)}+\bar{X_\tau}^{(s_i)} - \bar{X_0}^{(s_i)}\\
     = \int_0^T f(\bar{X_t}, t)dt - \frac{1}{2}\int_0^T g^2(t)\nabla_x\log p_t(\bar{X_t}|s_i))dt,
\end{align*}
where $s_i$ is the sensitive attribute of the latent variable $\bar{X}$, and $s_i\in\{s_0, s_1\}$. Then, by Equations \peqref{eq:s_t} and \peqref{eq:switching_ode}, 
\begin{align*}
    \hat{X_T}^{(s_1)} - \hat{X_0}^{(s_1)} = \bar{X_T}^{(s_0)} - \bar{X_\tau}^{(s_0)}+\bar{X_\tau}^{(s_1)} - \bar{X_0}^{(s_1)}.
\end{align*}
Note that we labeled the sensitive attribute of $\hat{X}$ as $s_1$.
Now, we aim to find $\tau$ such that $\hat{X_0}^{(s_1)}\overset{D}{=}\hat{X_0}^{(s_0)}$, where $\hat{X_T}^{(s_1)}\overset{D}{=}\hat{X_T}^{(s_0)} \sim \mathcal{N(\textbf{0}, \textbf{I})}$. 
Then, 
{\small\begin{align*}
\bar{X_0}^{(0)} - \bar{X_\tau}^{(0)} - (\bar{X_0}^{(1)} - \bar{X_\tau}^{(1)} )
= \frac{1}{2}\int_0^\tau D(t)dt
= \frac{1}{4}\int_0^T D(t)dt,
\end{align*}}
where \Eqref{eq:tau_sum} satisfies.
\end{proof}

\begin{T1}
(Restated) Assuming a pre-trained model is trained with \Eqref{eq:forward}, and the subsequent reverse-ODE represents an equivalent probability flow ODE corresponding to the pre-trained model \Eqref{eq:reverse_ode}. Then, the solution of \Eqref{eq:switching_ode} has the same distribution as the pre-trained ODE.
\end{T1}

\begin{proof} 
    Let the solution of conditional SDE \Eqref{eq:forward} be $X_t\sim p_t(X_t|S)$, and the solution of reverse SDE \eqref{eq:reverse} $\bar{X}_t\sim p_t(\bar{X}_t|S)$. Then, by \citet{anderson1982reverse}, the solution of \Eqref{eq:forward} and \Eqref{eq:reverse} solves the same Fokker-Plank equation, leading to the conclusion that $p_t(X_t|S) \overset{D}{=} p_t(\bar{X}_t|S)$. Similar to this previous work, we want to show that  $p_t({X}_t|S) \overset{D}{=} p_t(\hat{X}_t|S_t)$.
    Firstly, the probability density $p_t$ of random variable $X_t$ satisfies the Fokker-Plank equation as follow:
    \begin{align}\label{eq:FP}
    \begin{split}
    &\partial_t p_t(X_t|S)=\\& -\nabla_x\cdot(f(X_t, t)p_t(X_t|S))+\frac{1}{2}Tr(g(t)^T\nabla_x^2p_t(X_t|S)g(t))
    \end{split}
    \end{align}
    Then, let the solution of \Eqref{eq:switching_ode} $\hat{X}_t\sim p_t(\hat{X}_t|S_t)$, where $S_t$ is given by \Eqref{eq:s_t}. In order to establish $p_t(X_t|S) \overset{D}{=} p_t(\hat{X}_t|S_t)$ $\forall t$, we show that if $p_t(X_t|S)$ satisfies the Fokker-Planck equation, then $p_t(\hat{X}_t|S_t)$ satisfies the same equation. Since $P(X_t, S_t) = P(X_t, S_t, S = S_t) + P(X_t, S_t, S \neq S_t)$, we can easily show that ${p_t(S_t|X_t)}= p(S|X_t)I(t>\tau) + (1- p(S|X_t))I(t<\tau)$, so that we can
    analyze $p_t$ by dividing it into two distinct cases based on the value of $t$.\\
    Case 1. $t>\tau$\\
    The solution of reverse-ODE of the pre-trained model satisfies 
    the \Eqref{eq:FP}. Then, the left-hand side of 
    \Eqref{eq:FP} is
    \begin{align}\label{eq:LHSFP_orig}
    \begin{split}
        &\partial_t \frac{p(S|{X}_t)p_t({X}_t)}{p(S)}\\ &= \frac{1}{p(S)}\{(\partial_tp(S|{X}_t))p_t({X}_t) +  p(S|{X}_t)\partial_t p_t({X}_t)\}.
    \end{split}
    \end{align} 
    The same equation for $\hat{X}_t$ is
    \begin{align}\label{eq:LHSFP_switching}
    \begin{split}
    &\partial_t p_t(\hat{X}_t|S_t) =\partial_t \frac{p_t(S_t|\hat{X}_t)p_t(\hat{X}_t)}{p_t(S_t)} \\&= \frac{1}{1-p(S)}\{-(\partial_tp(S|\hat{X}_t))p_t(\hat{X}_t) +(1-  p(S|\hat{X}_t))\partial_t p_t(\hat{X}_t)\} \\&= \frac{1}{1-p(S)}\partial_t p_t(\hat{X}_t) - \frac{p(S)}{1-p(S)}\partial_t p_t(\hat{X}_t|S)
    \end{split}
    \end{align} 
    The right-hand side of \Eqref{eq:FP} for $\hat{X}_t$ is 
    \begin{align}
    \begin{split}
    &-\nabla_x\cdot\left(f(\hat{X}_t, t)\frac{(1-p_t(S|\hat{X}_t)) p(\hat{X}_t)}{1-p(S)}\right)\\
    &+\frac{1}{2}Tr\left(g(t)^T\nabla_x^2\left(\frac{(1-p_t(S|\hat{X}_t)) p(\hat{X}_t)}{1-p(S)}\right)g(t)\right)
    \end{split}\\
    \begin{split}
    =-\nabla_x\cdot\left(f(\hat{X}_t, t)(\frac{p_t(\hat{X}_t)}{1-p(S)} - \frac{p(S)}{1-p(S)}p_t(\hat{X}_t|S_t))\right)\\
    +\frac{1}{2}Tr\left(g(t)^T\nabla_x^2\left(\frac{p_t(\hat{X}_t)}{1-p(S)} - \frac{p(S)}{1-p(S)}p_t(\hat{X}_t|S)\right)g(t)\right)
    \end{split}\\\small
    \begin{split}\label{eq:RHSFP_switching}
    =&\frac{1}{1-p(S)}\left(-\nabla_x\cdot f(\hat{X}_t, t)(p_t(\hat{X}_t)+\frac{1}{2} Trg(t)^T\nabla_x^2{p_t(\hat{X}_t)}g(t)\right)\\
    &-\frac{p(S)}{1-p(S)} \cdot LHS \text{ of }\Eqref{eq:FP}
    \end{split}
    \end{align}
When substituting the $p_t(\hat{X}_t|S)$ that satisfies \Eqref{eq:FP} into \Eqref{eq:LHSFP_switching} and \Eqref{eq:RHSFP_switching}, we can easily demonstrate that  \Eqref{eq:LHSFP_switching} = \Eqref{eq:RHSFP_switching}.
    This means the two probability distribution is in distributionally same.

    By Case 1, for $t\leq\tau$, the reverse-ODE \Eqref{eq:reverse_ode} and \Eqref{eq:switching_ode} are the same equation, start from the same initial state. Therefore, $p_t({X}_t|S) \overset{D}{=} p_t(\hat{X}_t|S_t)$.
\end{proof}

\section{Ablation study}
\subsection{Additional Experiments for $\epsilon$-fairness}
In this section, we provide additional experiments for the CelebA dataset with the attribute ``Male", which is the imbalanced dataset. Training details are the same as mentioned in the Implementation details section. As shown in Table \ref{tab:epsilon_fairness}, even the imbalanced dataset, our method achieved best BER.\\ 
\begin{table}[!h]
\centering
\begin{tabular}{l|l|rrrr}
\hline
 & Methods & \multicolumn{1}{l}{$S = 0$} & \multicolumn{1}{c}{$S = 1$} & \multicolumn{1}{c}{gap} & \multicolumn{1}{c}{BER} \\\hline
\multicolumn{1}{c|}{} & Real & 0.90 & 0.95 & 0.05 & 0.93 \\\hline
\multirow{3}{*}{\begin{tabular}[c]{@{}c@{}}Syn\\$\rightarrow$\\Orig\end{tabular}} & Vanilla & 1.95 & 2.16 & 0.20 & 2.05 \\
 & Mixing & 38.13 & 37.63 & 0.50 & 37.88 \\
 & Ours & 34.97 & 47.75 & 12.78 & \textbf{41.36} \\\hline
\multirow{3}{*}{\begin{tabular}[c]{@{}c@{}}Orig\\$\rightarrow$\\ Syn\end{tabular}} & Vanilla& 6.92 & 3.85 & 3.07 & 5.38 \\
 & Mixing  & 42.85 & 27.81 & 15.04 & 35.33 \\
 & Ours& 61.13 & 56.43 & 4.70 & \textbf{58.78} \\\hline
\end{tabular}
 \caption{The error rates on CelebA with attribute ``Male" .}
\end{table}

\subsection{The stability of $\tau$-searching algorithm}
We analyze the results of \Eqref{eq:tau_diff} vary the batch size of the $\tau$-searching algorithm,  within the range $\{2, 4, 8, 16, 32, 64, 128, 256, 512, 768\}$, employing repeated sampling with a maximum batch size of 256. \Figref{fig:u_shape_batch} illustrates the results with FairFace dataset. In the figure, the transition point $\tau$ remains steady after the batch size is sufficiently large. This result supports that there is no need to compute the difference individually.

\begin{figure}
    \centering
    \includegraphics[width=0.8\linewidth]{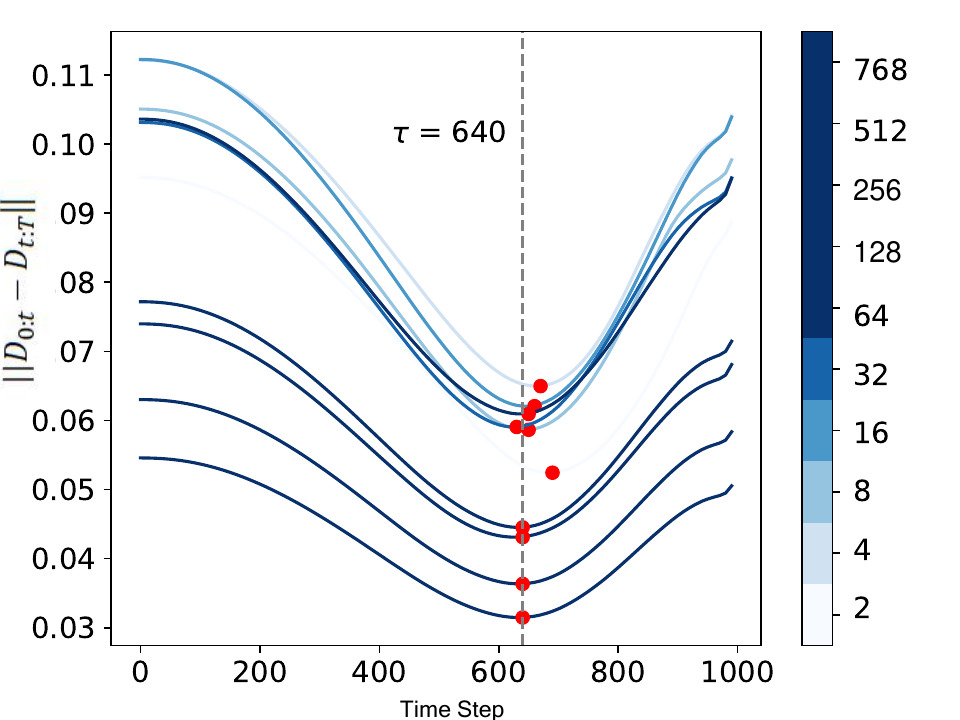}
    \captionof{figure}{Obtained optimal $\tau$ with different batch sizes.}
    \label{fig:u_shape_batch}
\end{figure}

\subsection{Dimension Reduction with CIFAR10}

\Figref{fig:pca_app} illustrate the dimension reduction (PCA, $n=2$) results for both vanilla sampling and the proposed methods are presented for Bird-Truck, Bird-Frog, and Frog-Truck pairs within the CIFAR-10 datasets. We compute principal component vectors using vanilla sampling for each dataset pair. The proposed methods show a large portion of overlapped parts.

\begin{figure*}[!h]
\centering     
 \subfigure[Vanilla (Bird-Truck)]{\label{fig:bird_truck_vanilla}\includegraphics[width=55mm]{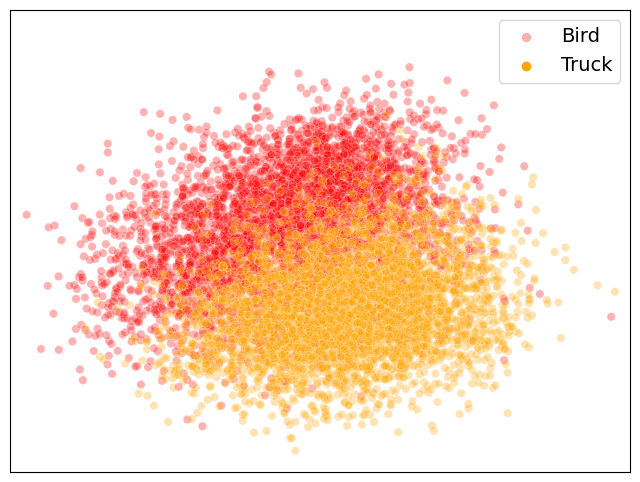}}
    \subfigure[Vanilla (Bird-Frog)]{\label{fig:bird_frog_vanilla}\includegraphics[width=55mm]{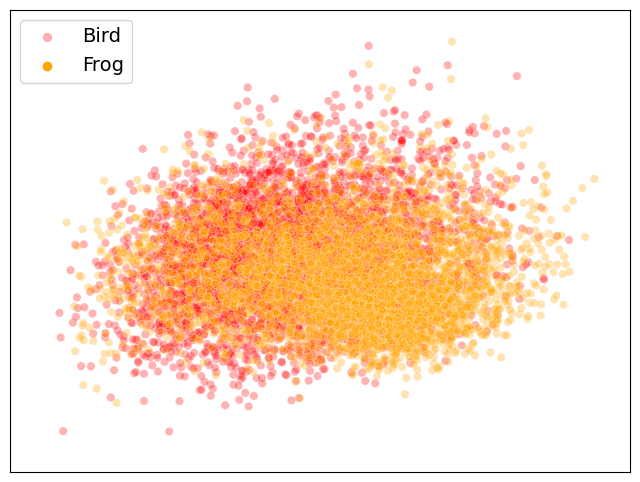}}
    \subfigure[Vanilla (Frog-Truck)]{\label{fig:frog_truck_vanilla}\includegraphics[width=55mm]{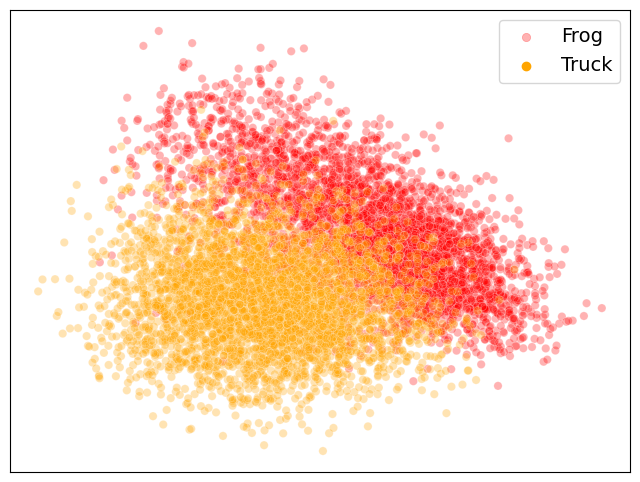}}
    
    \subfigure[Ours (Bird-Truck)]{\label{fig:bird_truck_ours}\includegraphics[width=55mm]{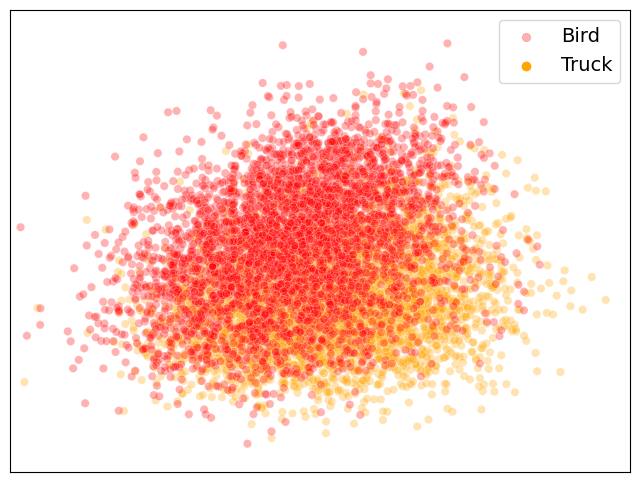}}
    \subfigure[Ours (Bird-Frog)]{\label{fig:bird_frog_ours}\includegraphics[width=55mm]{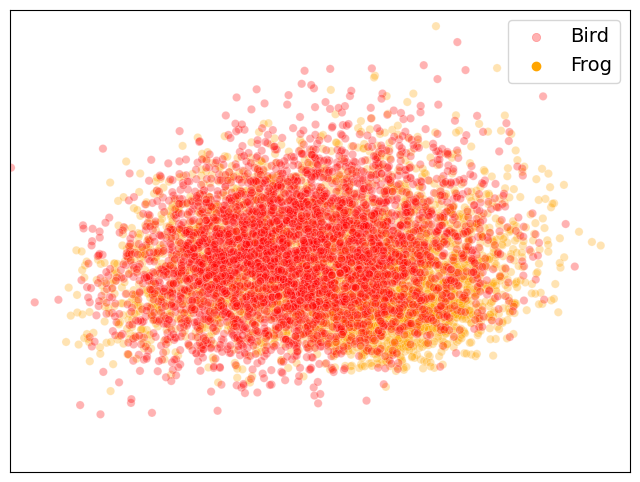}}
    \subfigure[Ours (Frog-Truck)]{\label{fig:frog_truck_ours}\includegraphics[width=55mm]{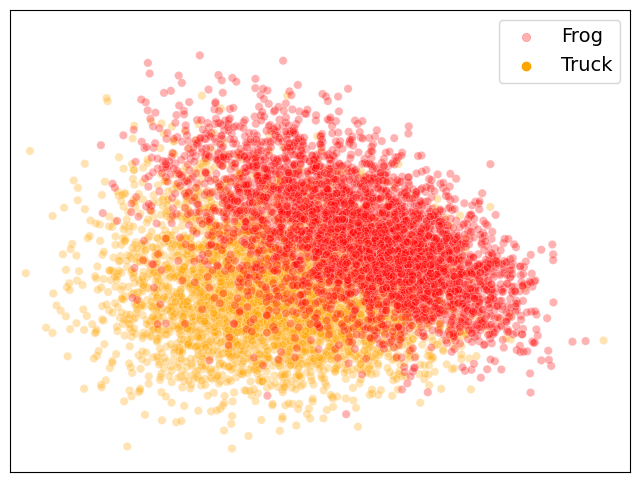}}
    \caption{Comparison of dimension reduction (PCA, $n=2$) results for both vanilla sampling and the proposed methods. The datasets are the same as in \Figref{fig:pca}, i.e., Bird-Truck, Bird-Frog, and Frog-Truck pairs within the CIFAR-10 datasets.}
    \label{fig:pca_app}
\end{figure*}

\subsection{Image Captioning}

\begin{figure}
    \centering
    \includegraphics[width=\linewidth]{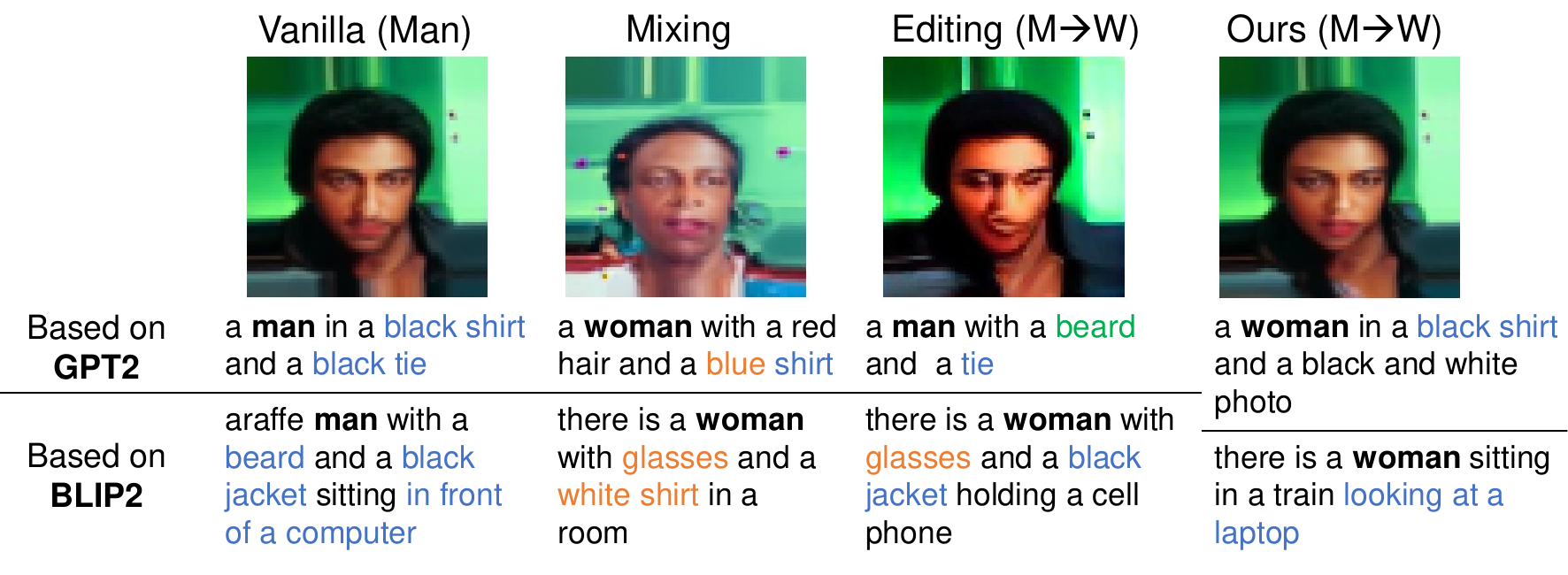}
    \captionof{figure}{Visualization and captioning of sampled images with LLM models.}
    \label{fig:caption}
\end{figure}
To verify that the switching mechanism maintains the original image distribution except for the sensitive attribute, we captioned the sampled image with LLMs (GPT2 and BLIP2).
\Figref{fig:caption} illustrates the images sampled from the diffusion model trained with the CelebA dataset. Blue represents the retained original context, while orange represents differences. Green denotes features not present in the original captioning but existing in the original image. Our method effectively preserves other features, with high quality. 
\subsection{Visualization}
\paragraph{Stable diffusion}

\Figref{fig:figure_stablediffusion} illustrates the generated images using text conditioning: \{\textit{``young person"}, \textit{``old person"}\}, \{\textit{``man"}, \textit{``woman"}\}, \{\textit{``black firefighter"}, \textit{``white firefighter"}\} with $\tau \in \{0,650,700,750,800,850,900,1000\}$ from left to right. The generated images retain qualities of pre-trained model without additional training while achieving fair images which have coarse part similar to $s_0$ but generate images based on $s_1$ for distributional fairness.

\begin{figure*}[!ht]
\centering     
    \includegraphics[width=140mm]{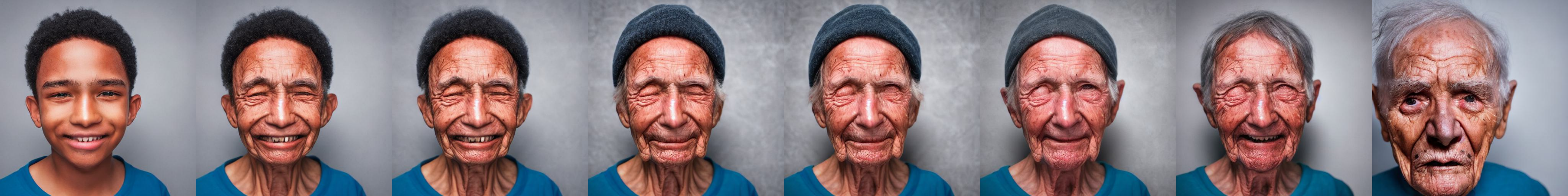}
    \includegraphics[width=140mm]{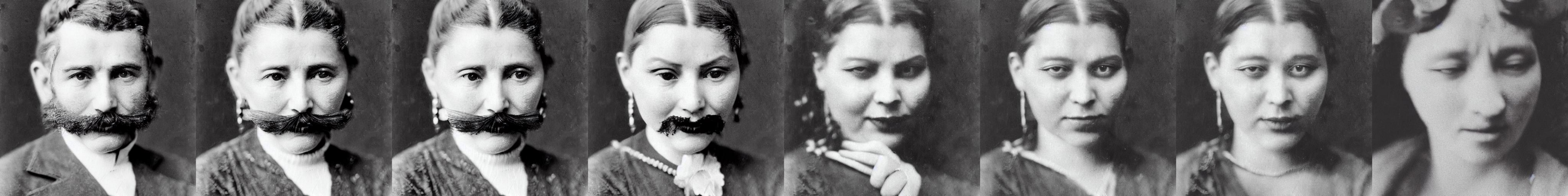}
    \includegraphics[width=140mm]{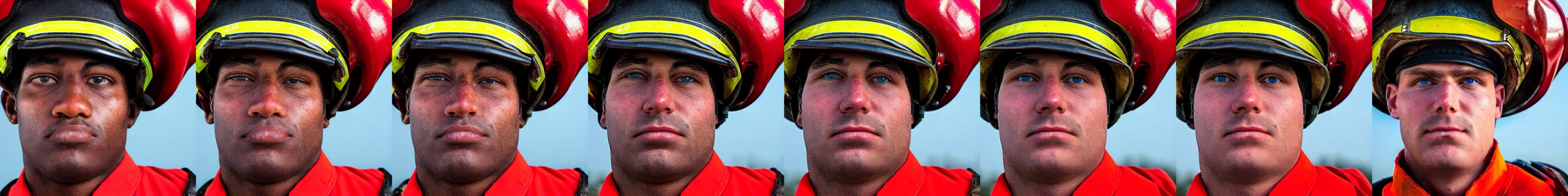}
    \caption{Generated fair images via Attribute Switching mechanism with $\tau \in \{0,650,700,750,800,850,900,1000\}$. $\tau=0$ (left) indicates vanilla sampling with $s_0$ and $\tau=1000$ (right) indicates vanilla sampling with $s_1$. While our method yield high quality images, it also achieves fairness.}
    \label{fig:figure_stablediffusion}
\end{figure*}
\paragraph{CelebA}
We visualize the sampled images with fixed seed = 0. 
Sampling with the diffusion model trained with the settings mentioned in the last section. As the diffusion model trained only 100 epochs, the model may not fully learn the total manifold of embedding space. Here, we can see that the model which is not fully trained, has lower-quality images sampled with mixing.\\
While attribute switching involves switching an attribute while maintaining the original image, mixing embedding can generate different images. 
More importantly in many cases, images sampled with a label of 1 and a probability of 0.6 were nearly identical to images with a label of 0 and the same probability. This means that we need to sample across the range of desired $p$ values to find the suitable one. On the other hand, attribute switching allows only one sampling attempt after storing the value.

\begin{figure*}[!ht]
\centering    
    \subfigure[Vanilla]{\includegraphics[width=55mm]{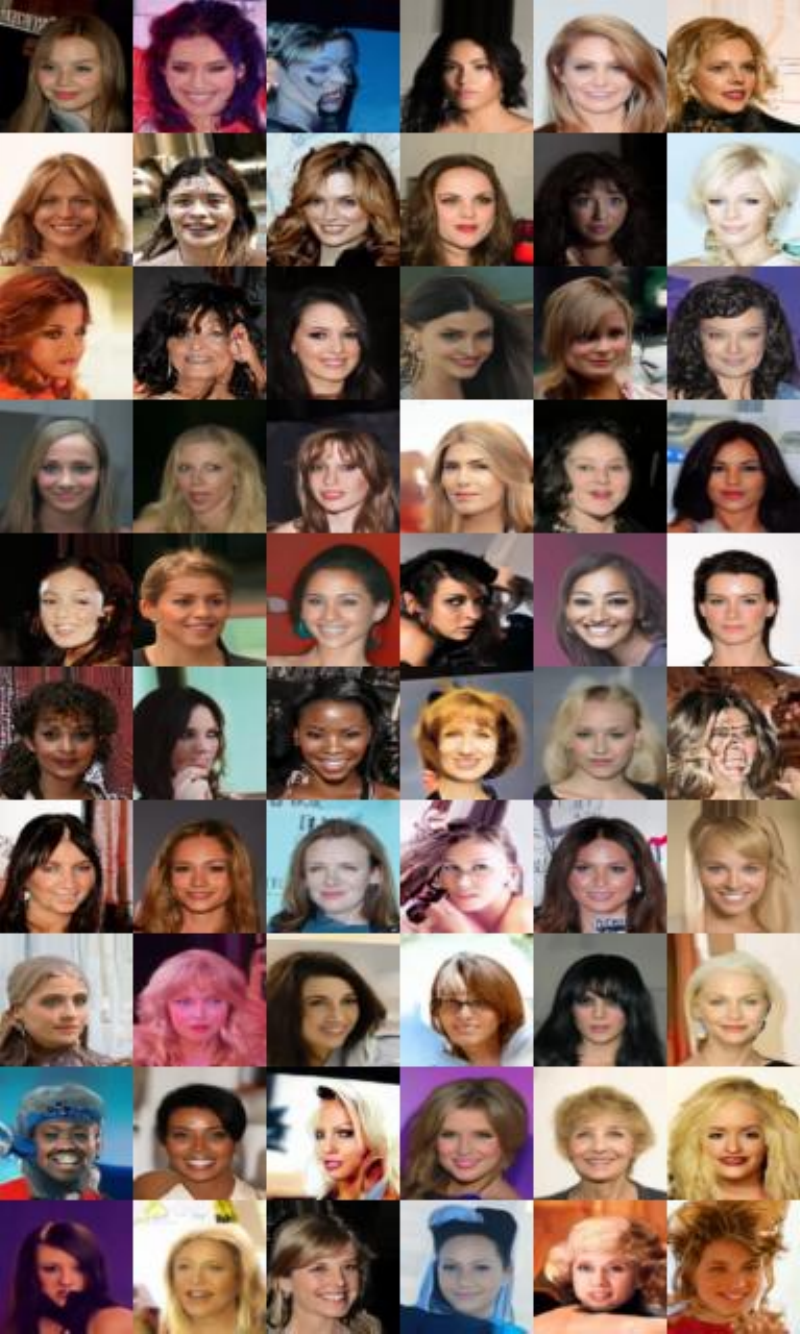}}
    \subfigure[Mixing with $p=0.6$]{\includegraphics[width=55mm]    {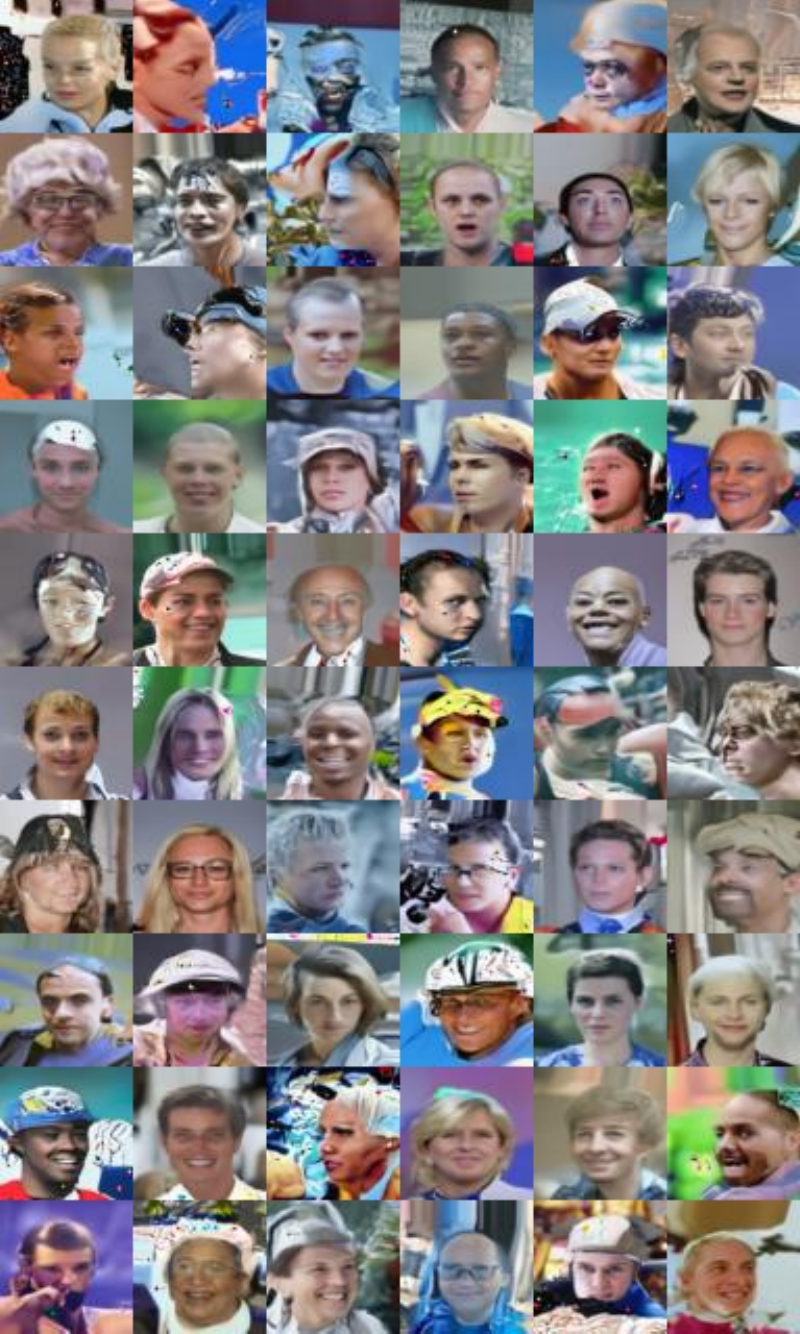}}
    \subfigure[Switching with $\tau = 640$]{\includegraphics[width=55mm]{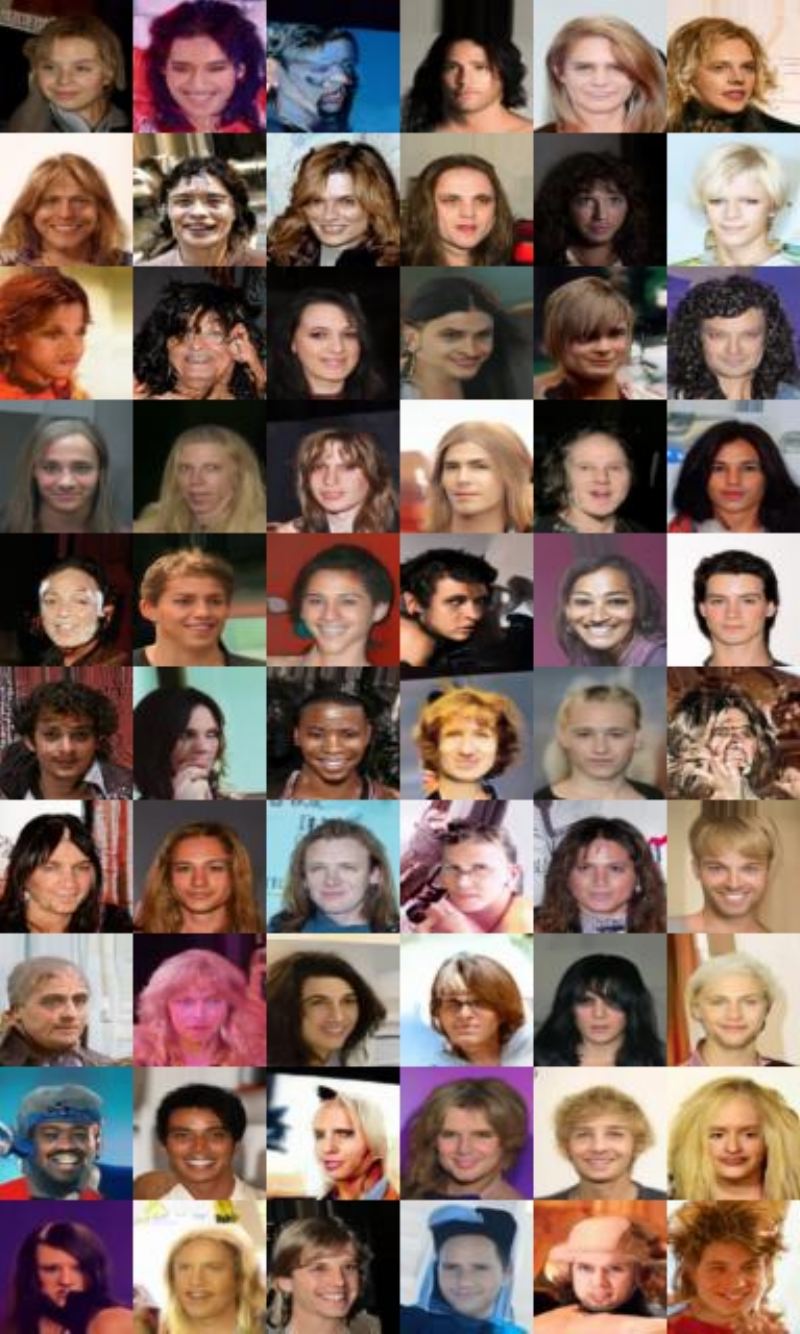}}
    \caption{Mixing and Switching on CelebA data with the Male attribute: Female to Male}
\end{figure*}

\begin{figure*}[!ht]
\centering    
    \subfigure[Vanilla]{\includegraphics[width=55mm]{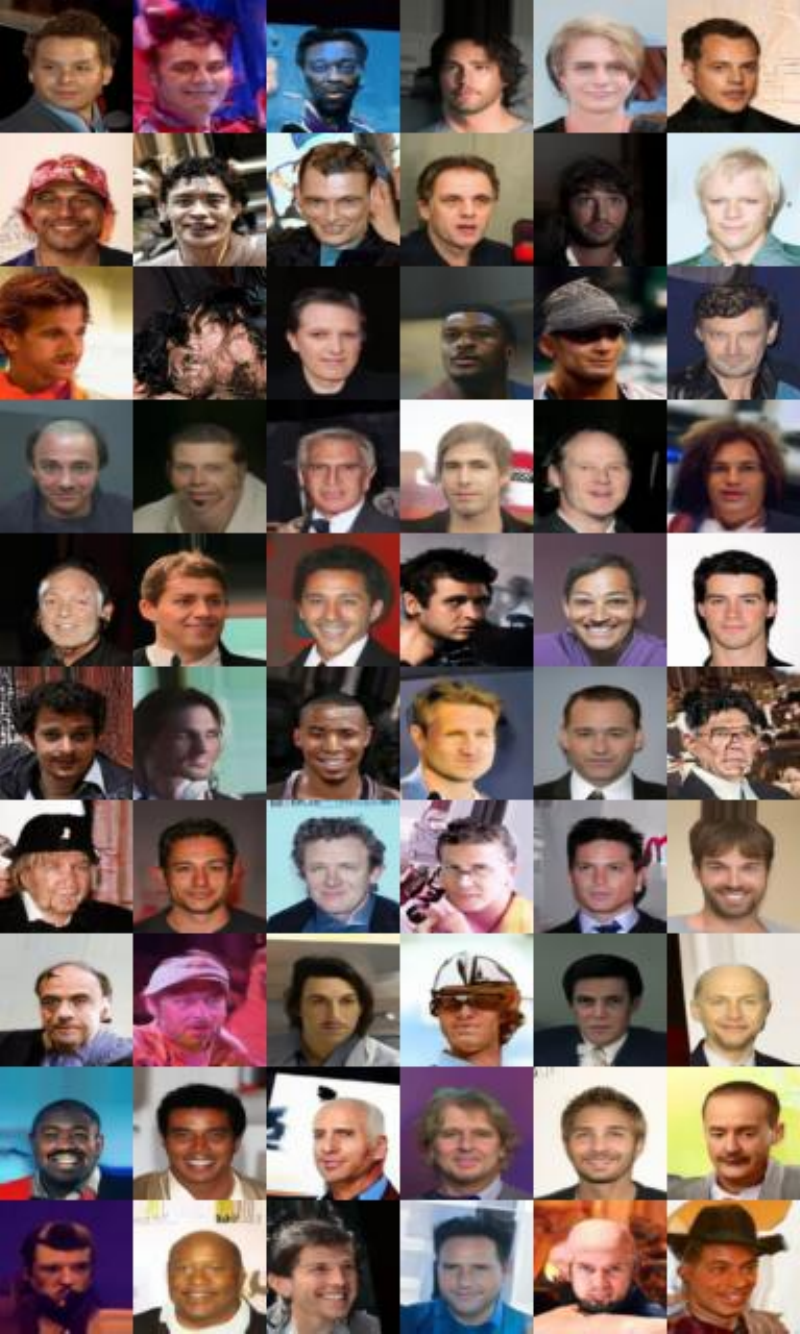}}
    \subfigure[Mixing with $p=0.6$]{\includegraphics[width=55mm]    {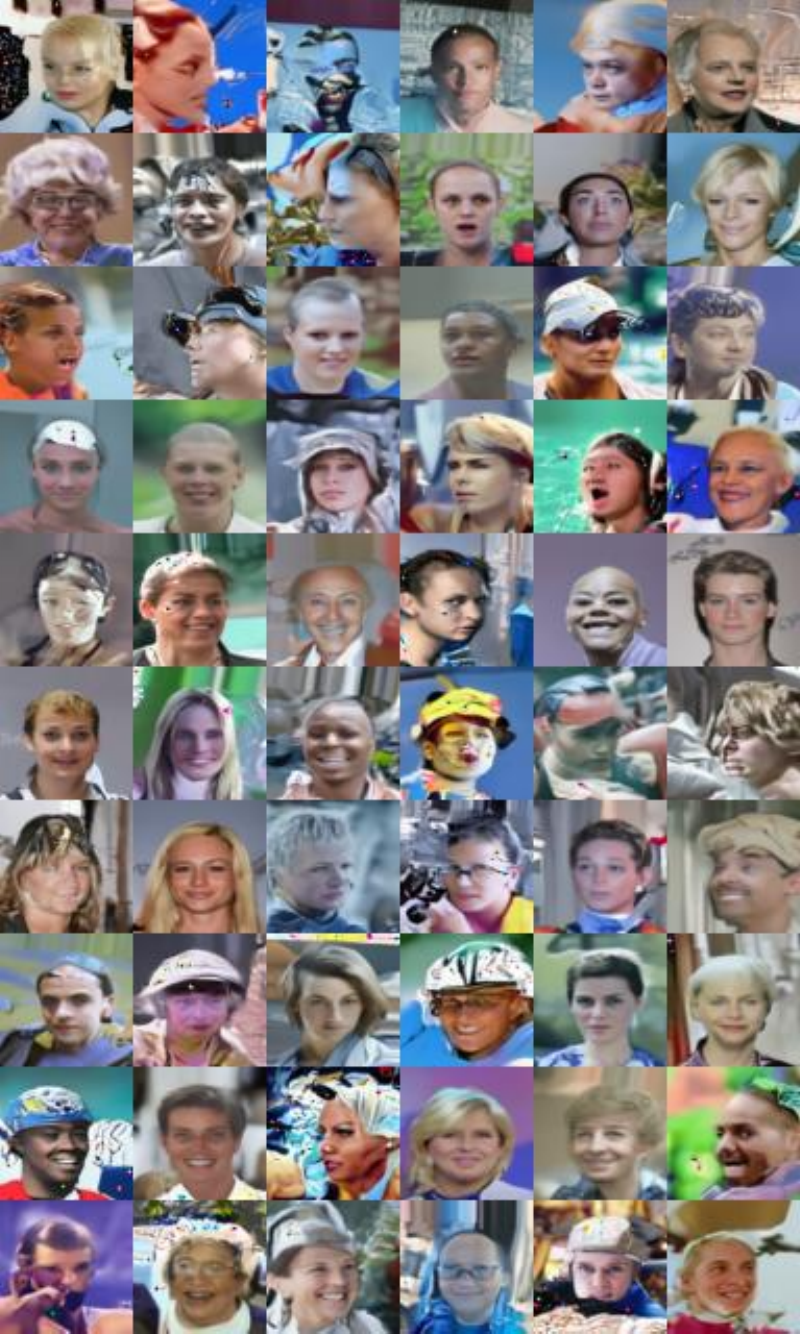}}
    \subfigure[Switching with $\tau = 640$]{\includegraphics[width=55mm]{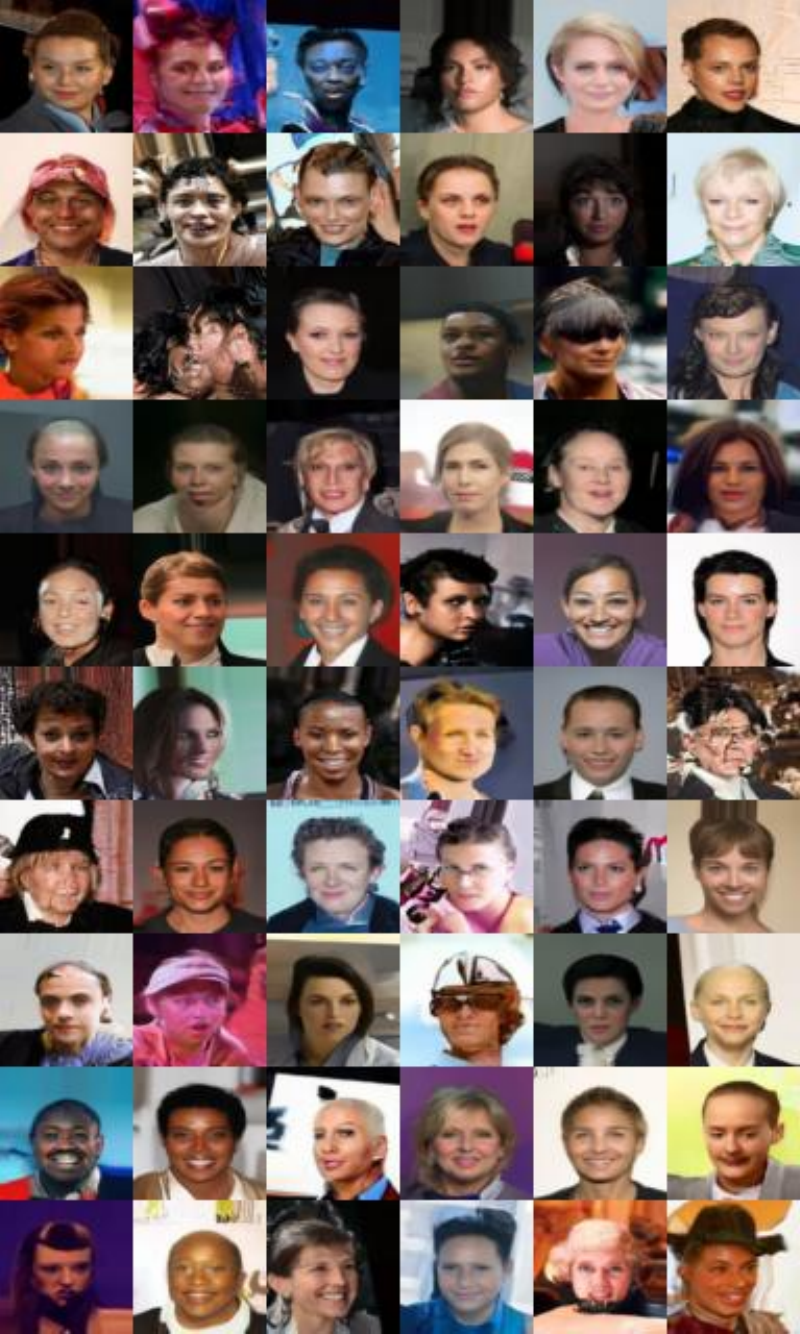}}
    \caption{Mixing and Switching on CelebA data with the Male attribute: Male to Female}
\end{figure*}

\begin{figure*}[!ht]
\centering    
    \subfigure[Vanilla]{\includegraphics[width=55mm]{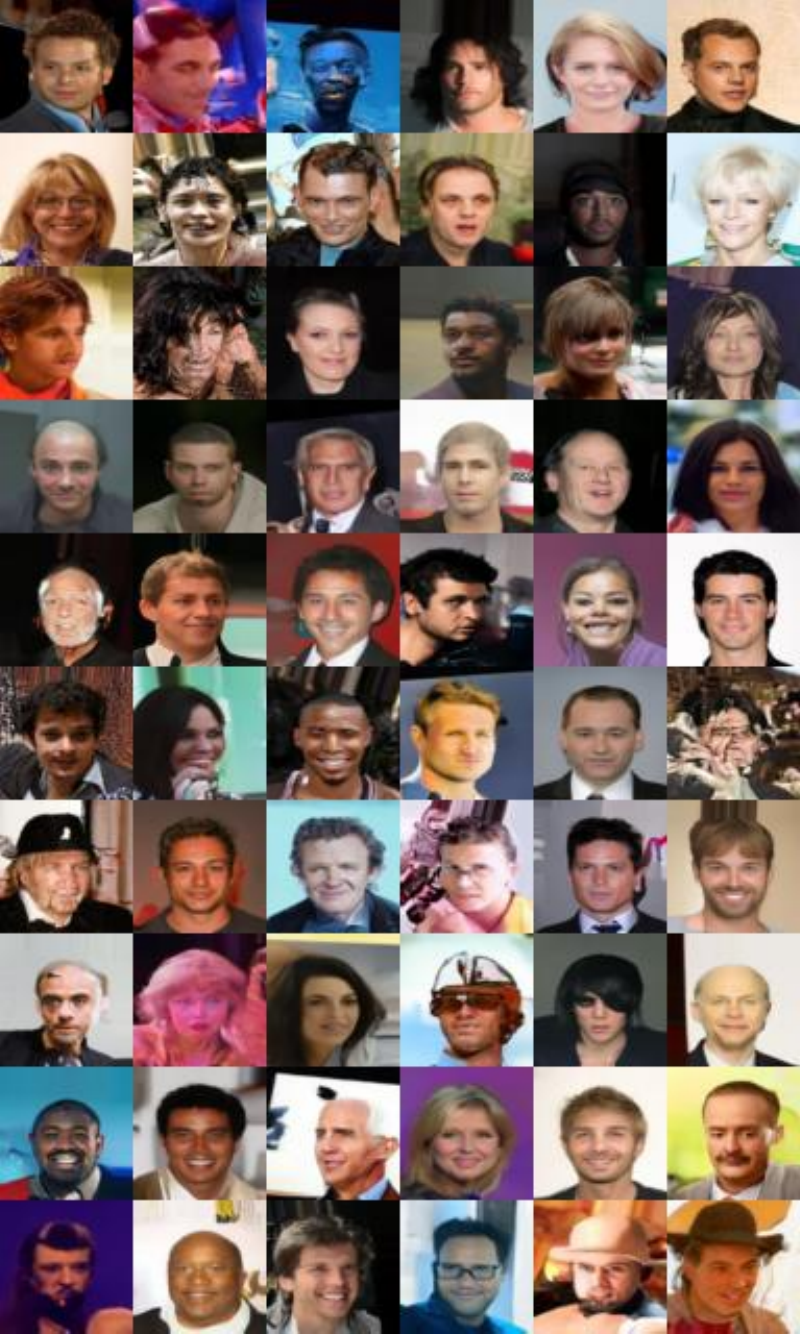}}
    \subfigure[Mixing with $p=0.6$]{\includegraphics[width=55mm]    {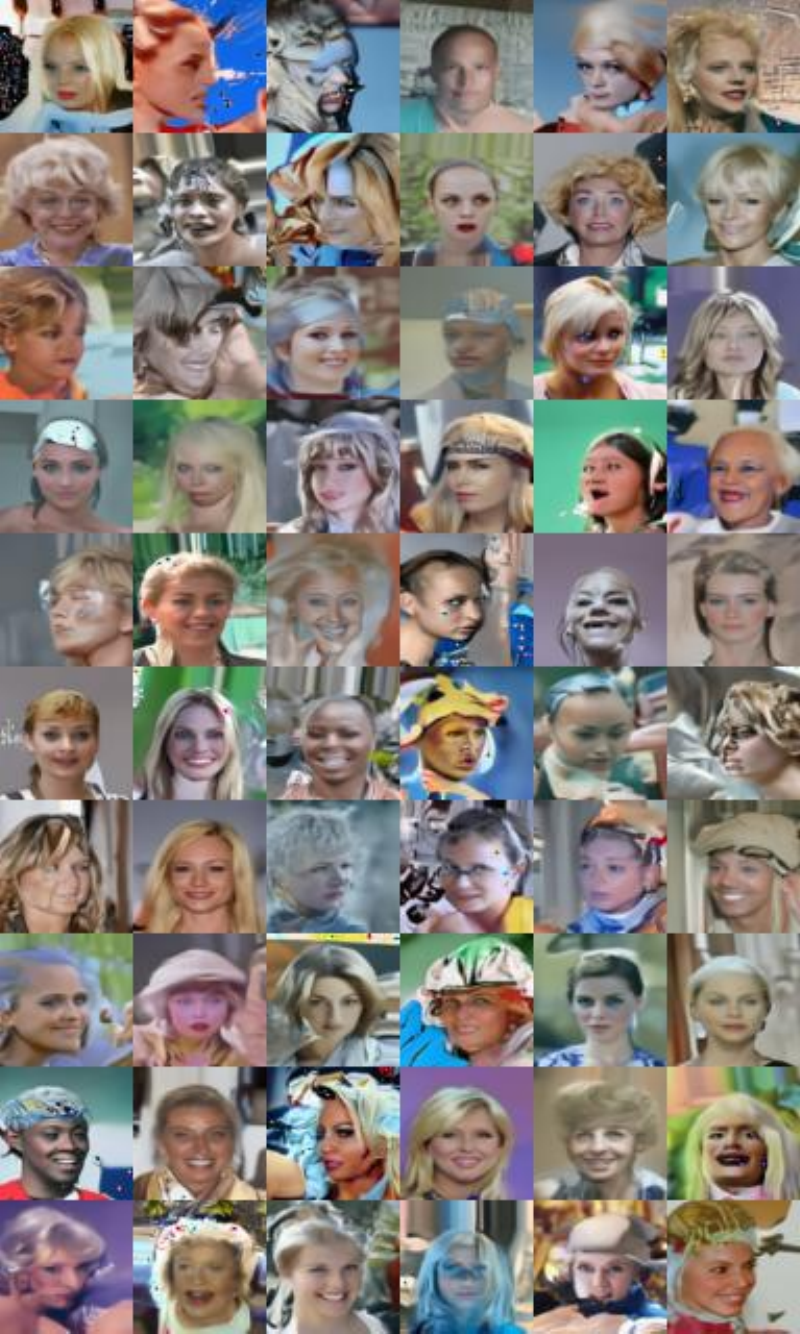}}
    \subfigure[Switching with $\tau = 640$]{\includegraphics[width=55mm]{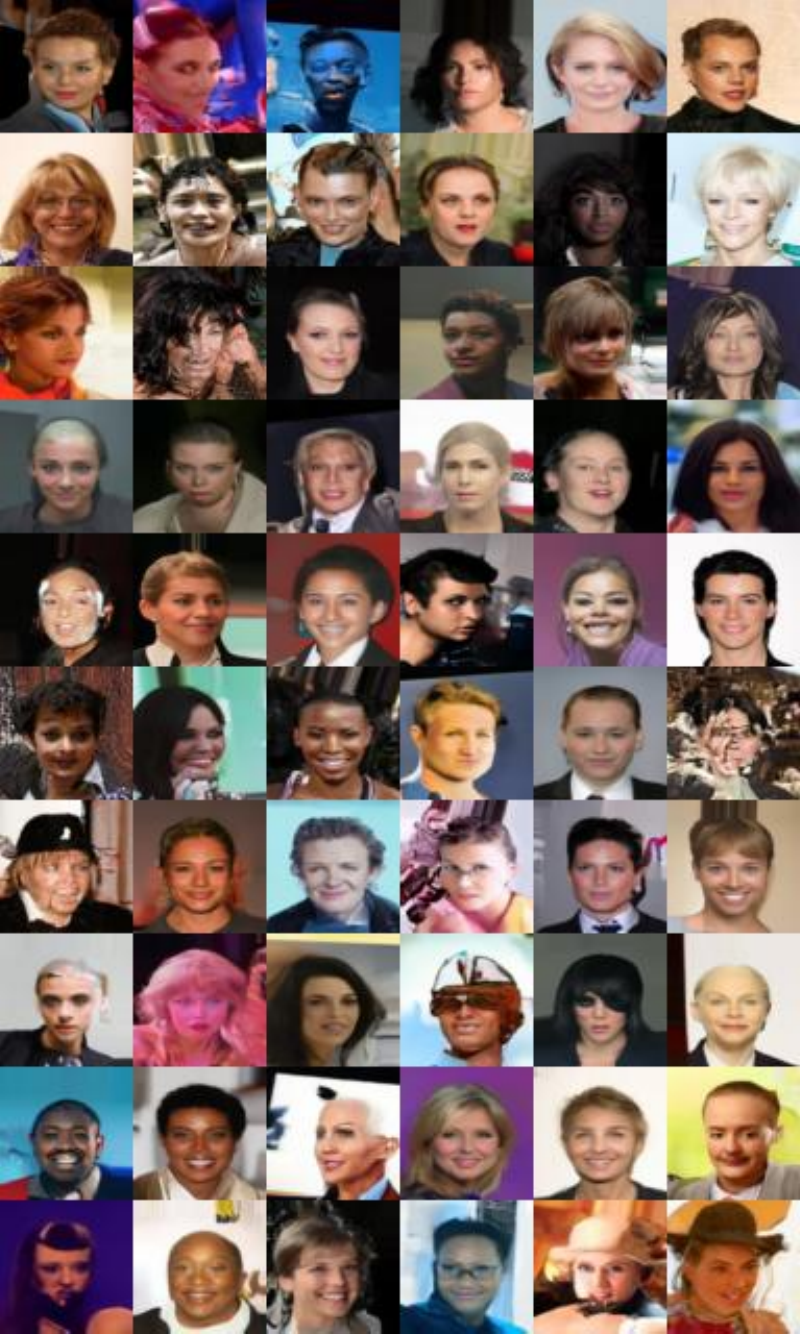}}
    \caption{Mixing and Switching on CelebA data with the Heavy-makeup attribute: No-makeup to Makeup}
\end{figure*}

\begin{figure*}[!ht]
\centering    
    \subfigure[Vanilla]{\includegraphics[width=55mm]{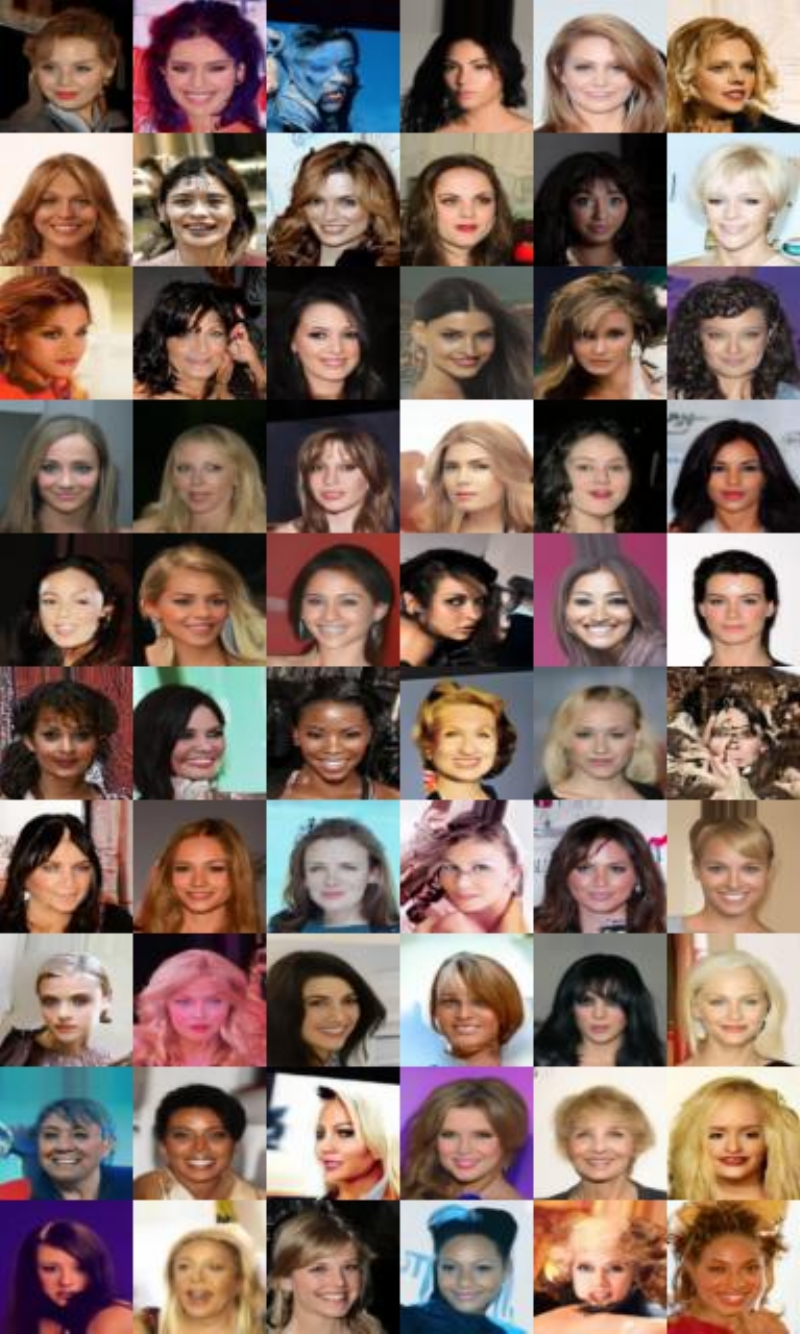}}
    \subfigure[Mixing with $p=0.6$]{\includegraphics[width=55mm]    {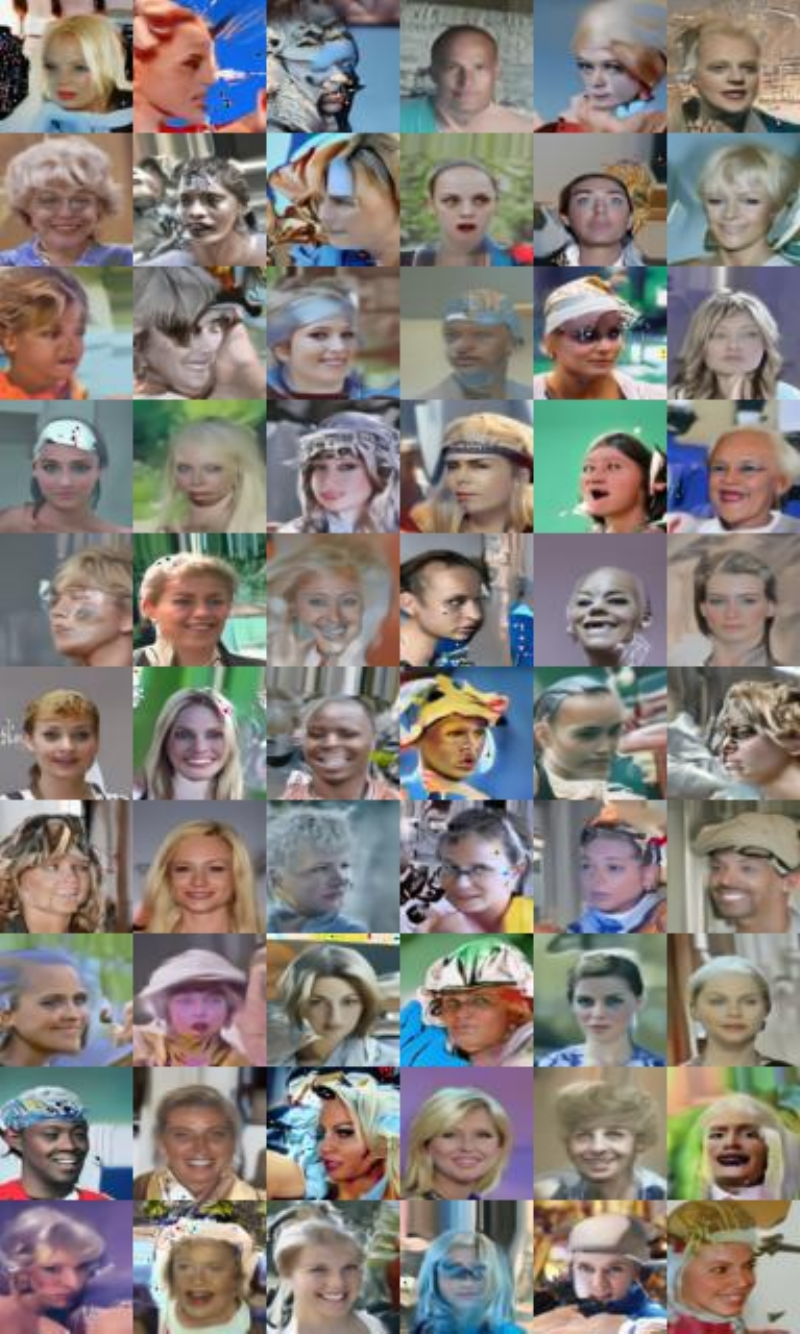}}
    \subfigure[Switching with $\tau = 640$]{\includegraphics[width=55mm]{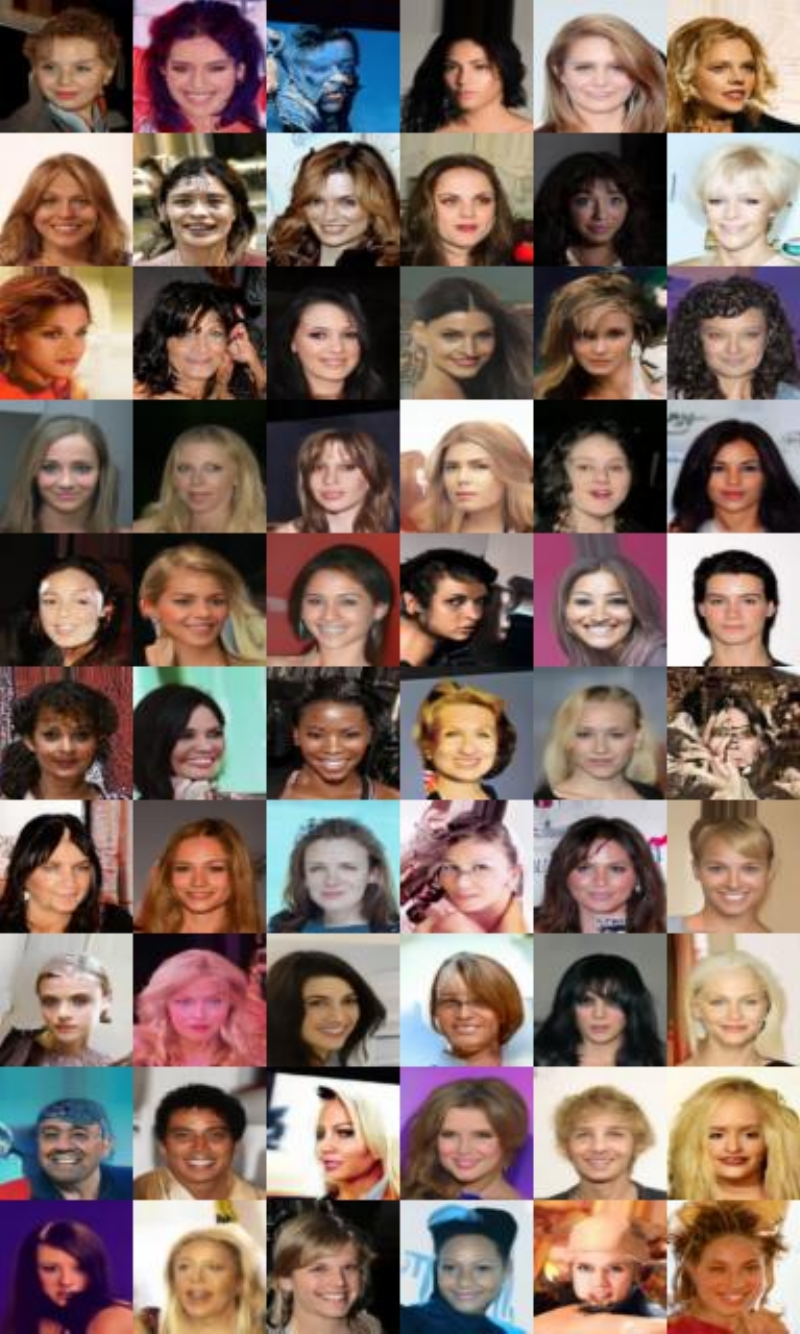}}
    \caption{Mixing and Switching on CelebA data with the Heavy-makeup attribute: Makeup to No-makeup}
\end{figure*}

\paragraph{FairFace}
Same as CelebA, we visualize the sampled images with fixed seed = 0. 
For the diffusion model trained on the FairFace dataset, some of the original samplings resulted in the generation of low-quality images. However, when employing the mixing, high-quality images were achieved. In the case of attribute switching, due to its ability to incorporate more characteristics of the original image, it tends to exhibit lower quality for images with inherently poor quality. 

\begin{figure*}[!ht]
\centering    
    \subfigure[Vanilla]{\includegraphics[width=55mm]{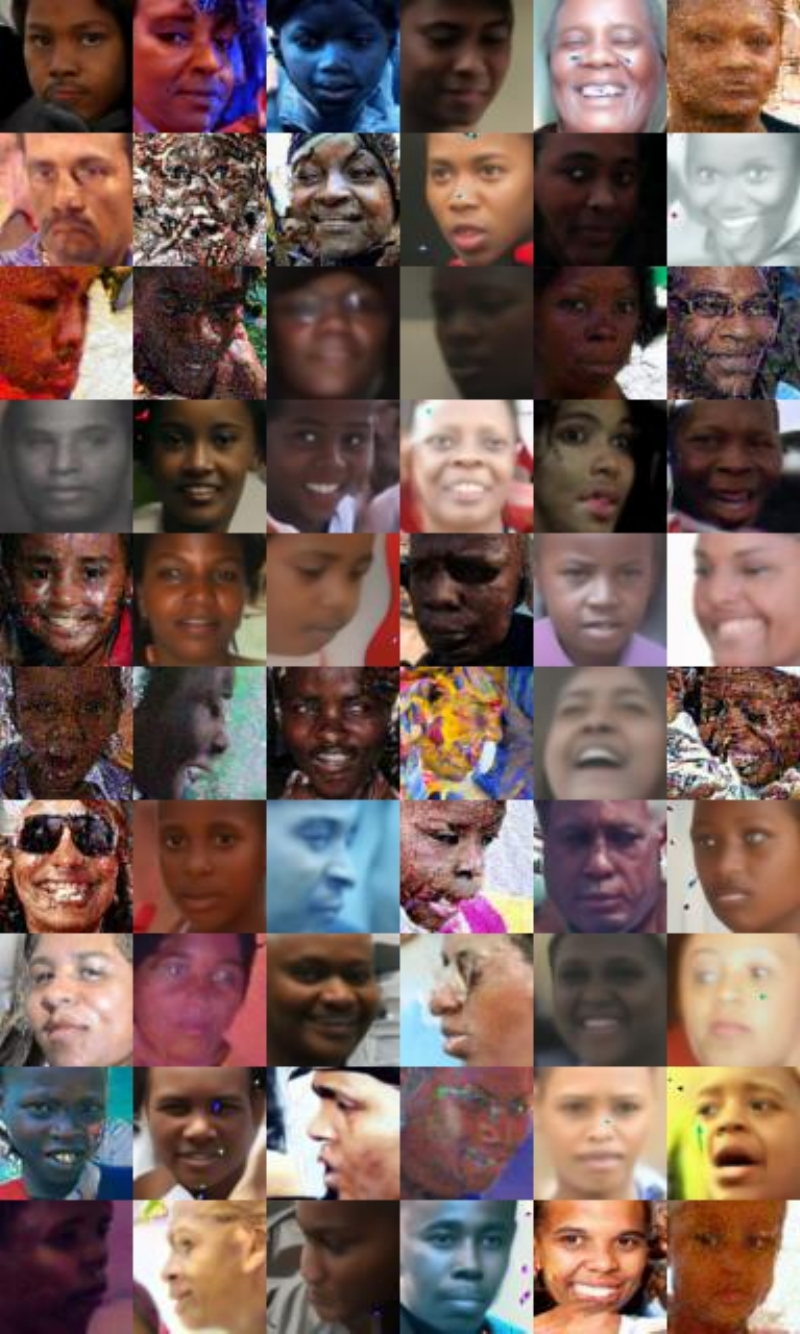}}
    \subfigure[Mixing with $p=0.6$]{\includegraphics[width=55mm]    {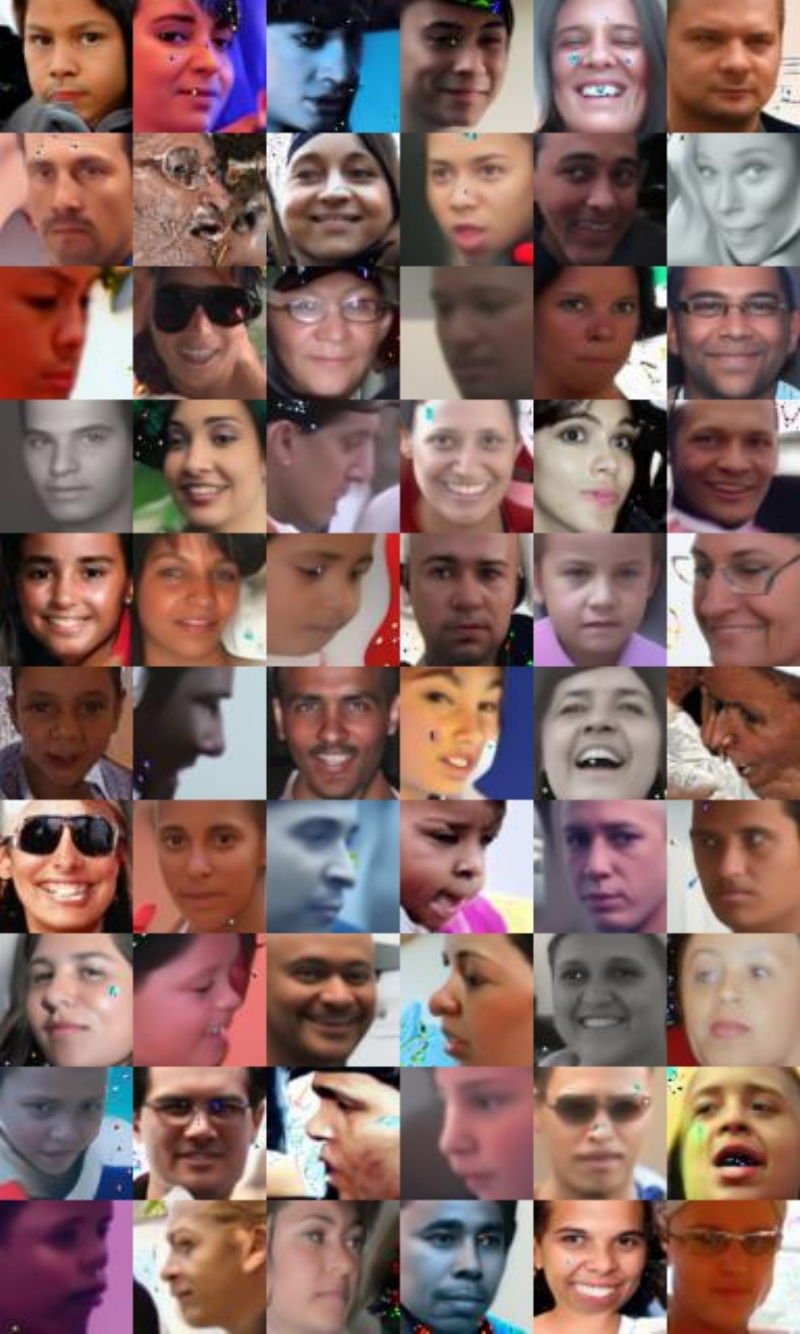}}
    \subfigure[Switching with $\tau = 640$]{\includegraphics[width=55mm]{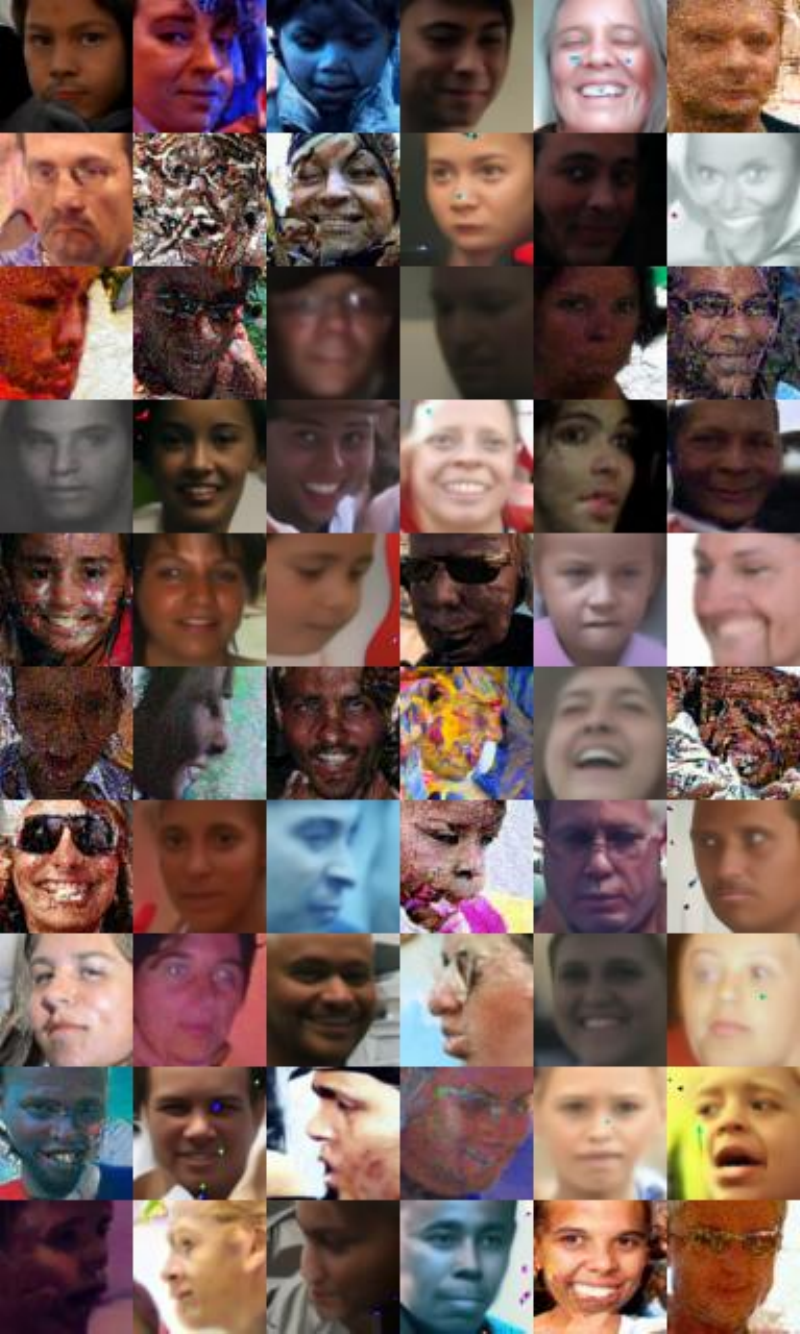}}
    \caption{Mixing and Switching on FairFace data with the Race attribute: Black to White}
\end{figure*}

\begin{figure*}[!ht]
\centering    
    \subfigure[Vanilla]{\includegraphics[width=55mm]{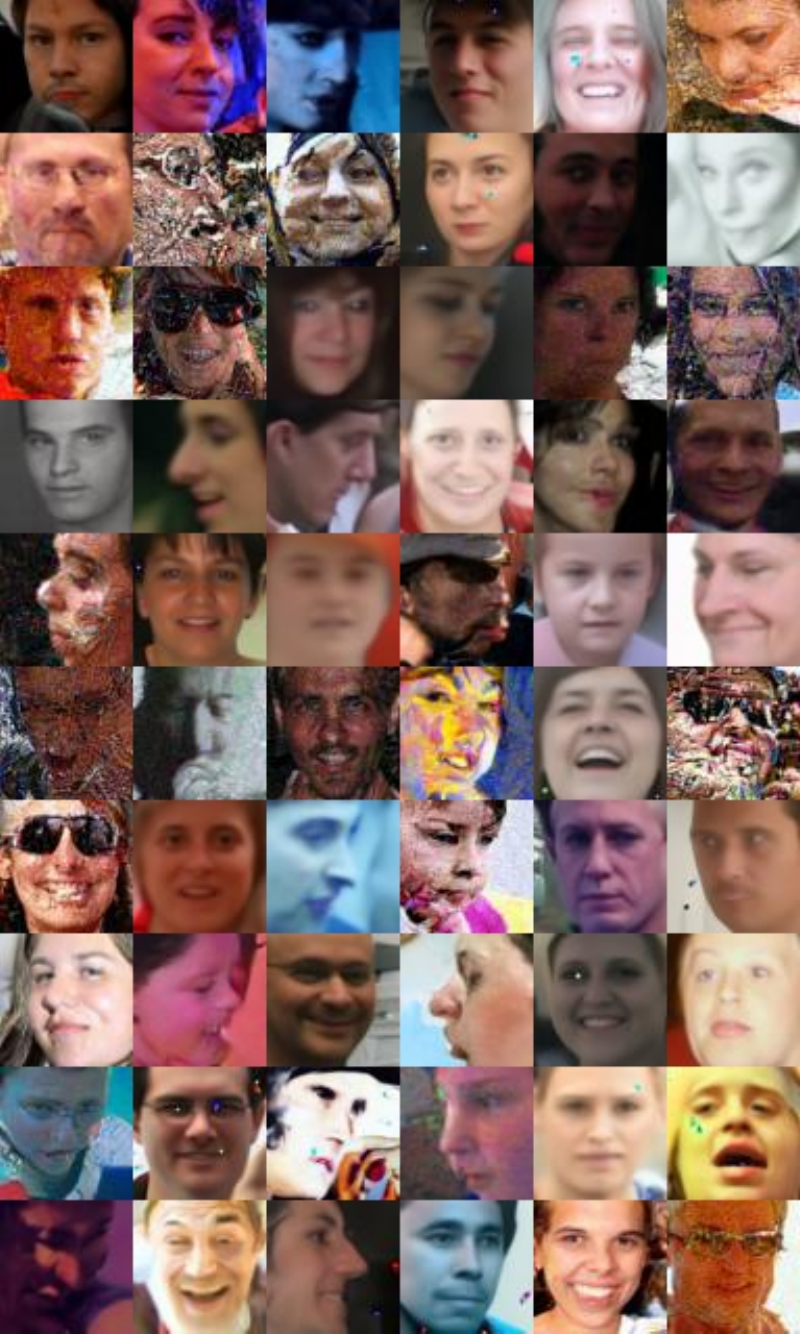}}
    \subfigure[Mixing with $p=0.6$]{\includegraphics[width=55mm]    {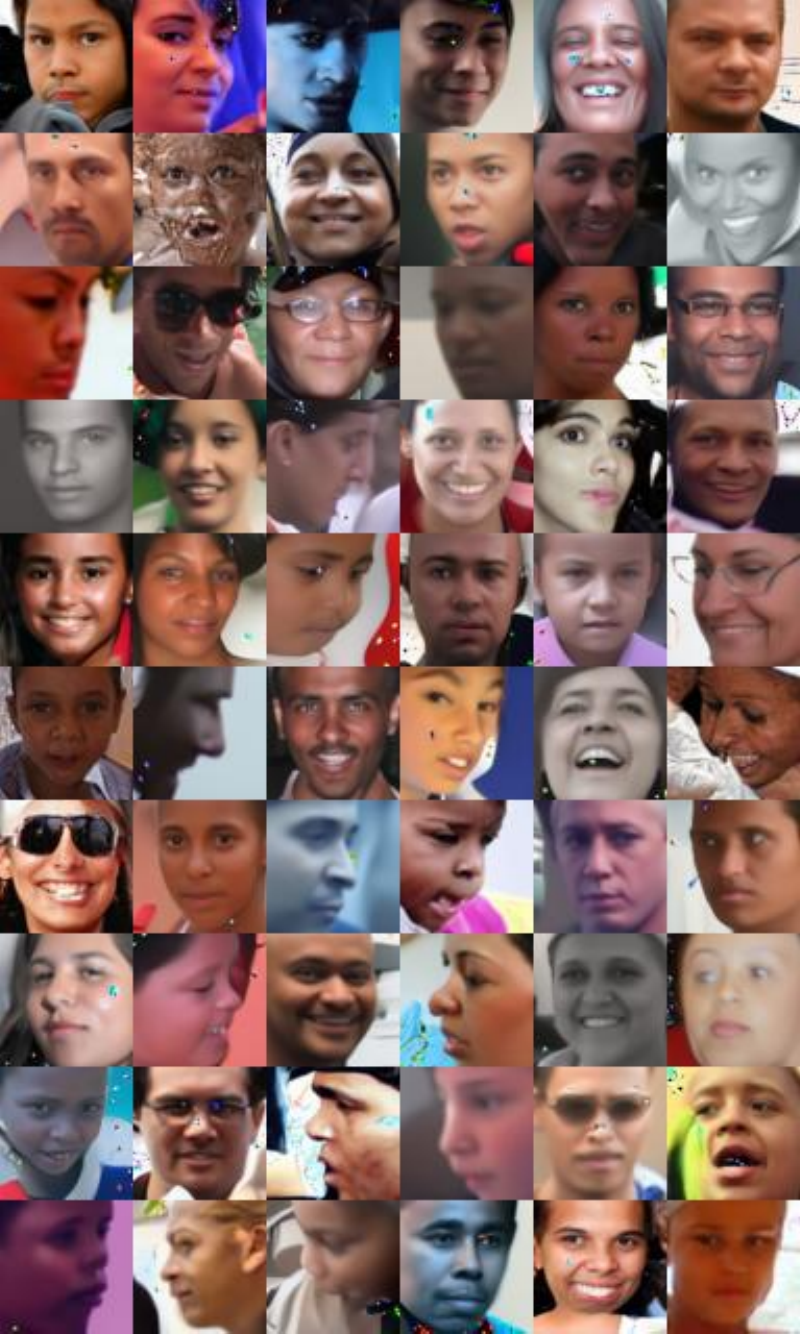}}
    \subfigure[Switching with $\tau = 640$]{\includegraphics[width=55mm]{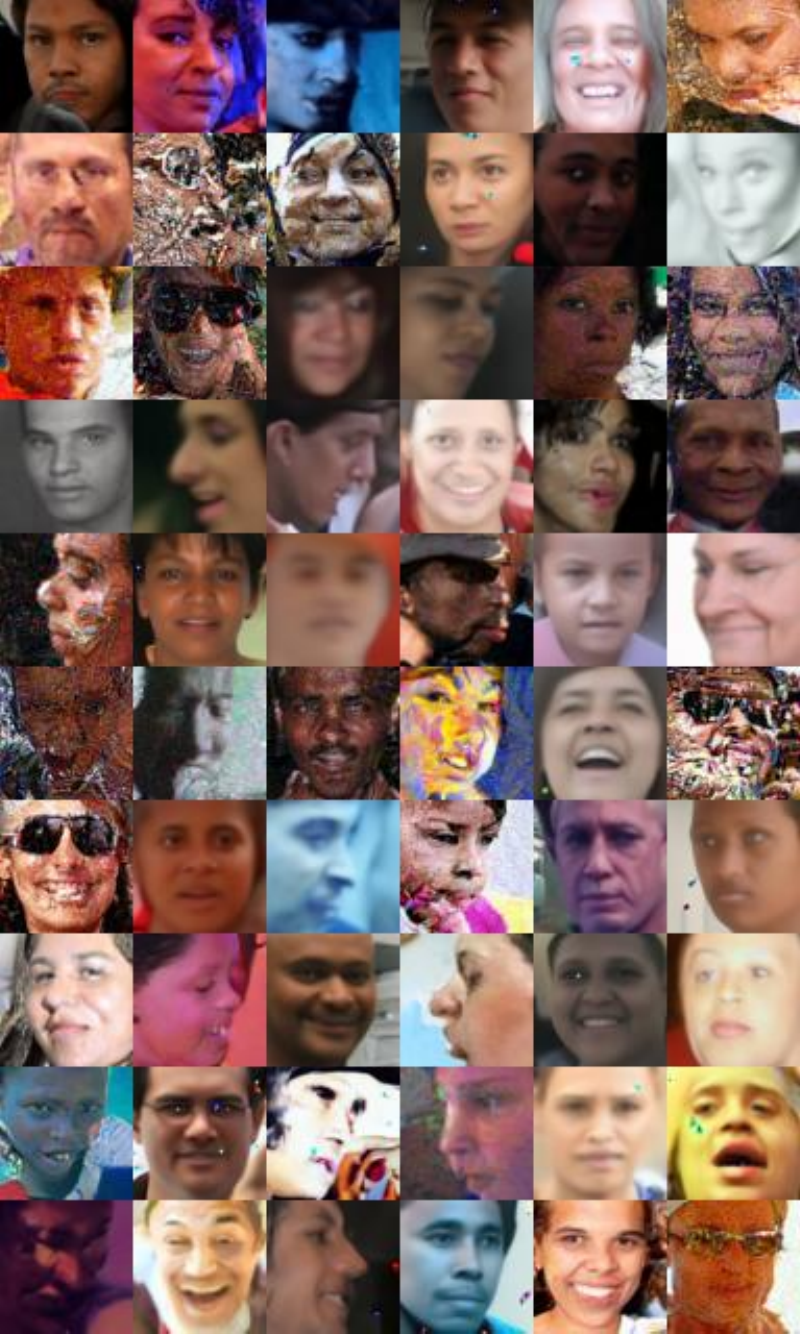}}
    \caption{Mixing and Switching on FairFace data with the Race attribute: White to Black}
\end{figure*}

\paragraph{Text Embedding}
Using the stable diffusion model \cite{rombach2022high} with text embedding, we compared the results of mixing embedding with $p \in \{0, 0.3, 0.4, 0.45, 0.5, 0.55, 0.6, 1.0\}$ and the proposed attribute switching using $\tau \in \{0, 650, 700, 750, 800, 850, 900, 1000\}$ in Figures \ref{fig:young_old}, \ref{fig:man_woman}.
The experimental settings are identical in \Figref{fig:figure_stablediffusion}. While mixing embedding to a certain range does not guarantee private results (e.g., mixing $p=0.5$ generates biased images), the proposed methods can generate  fair images (e.g., $\tau=700$) while preserving other attributes.

\begin{figure*}[!ht]
\centering    
    \subfigure[Mixing]{\includegraphics[width=80mm]{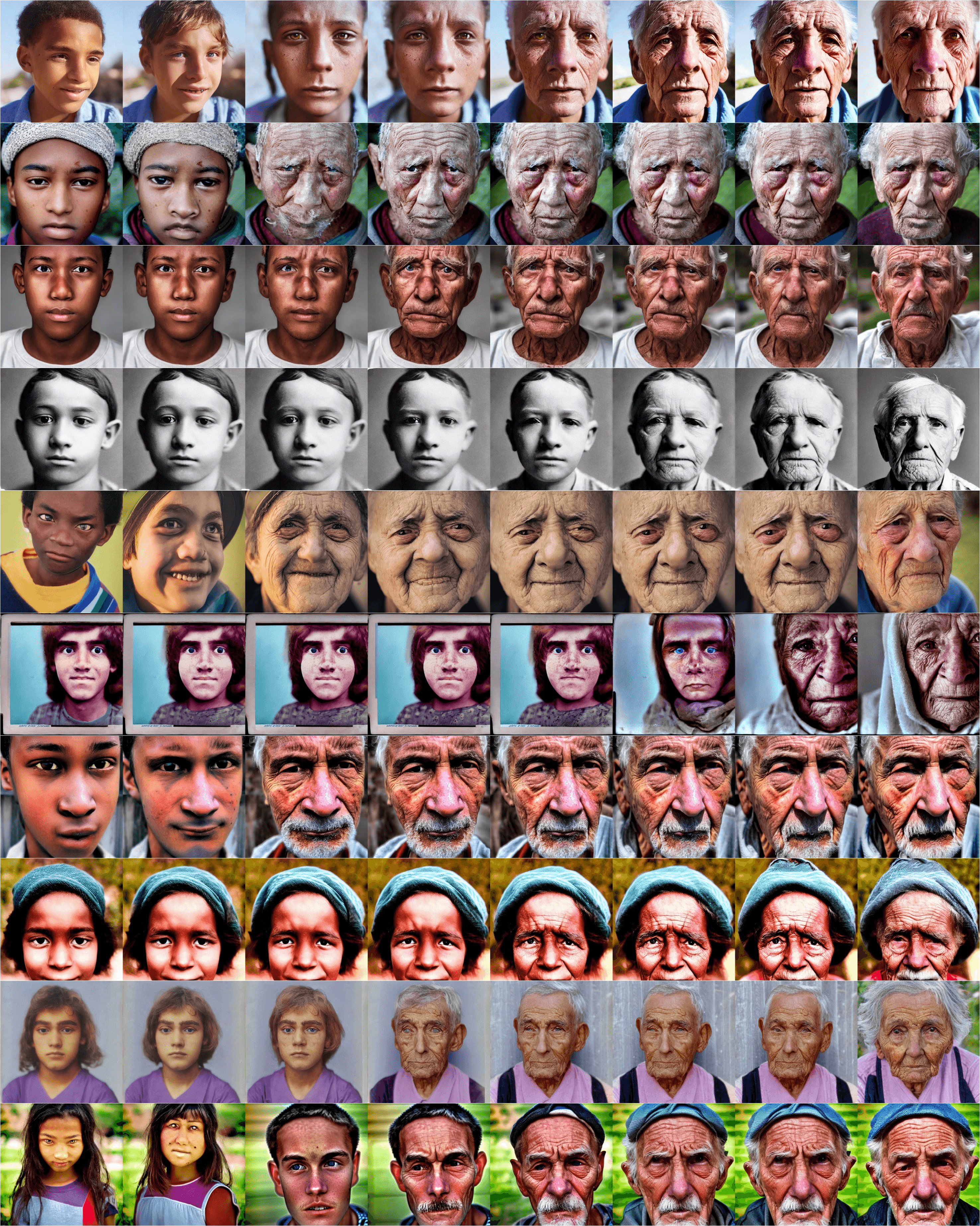}}
    \subfigure[Ours]{\includegraphics[width=80mm]{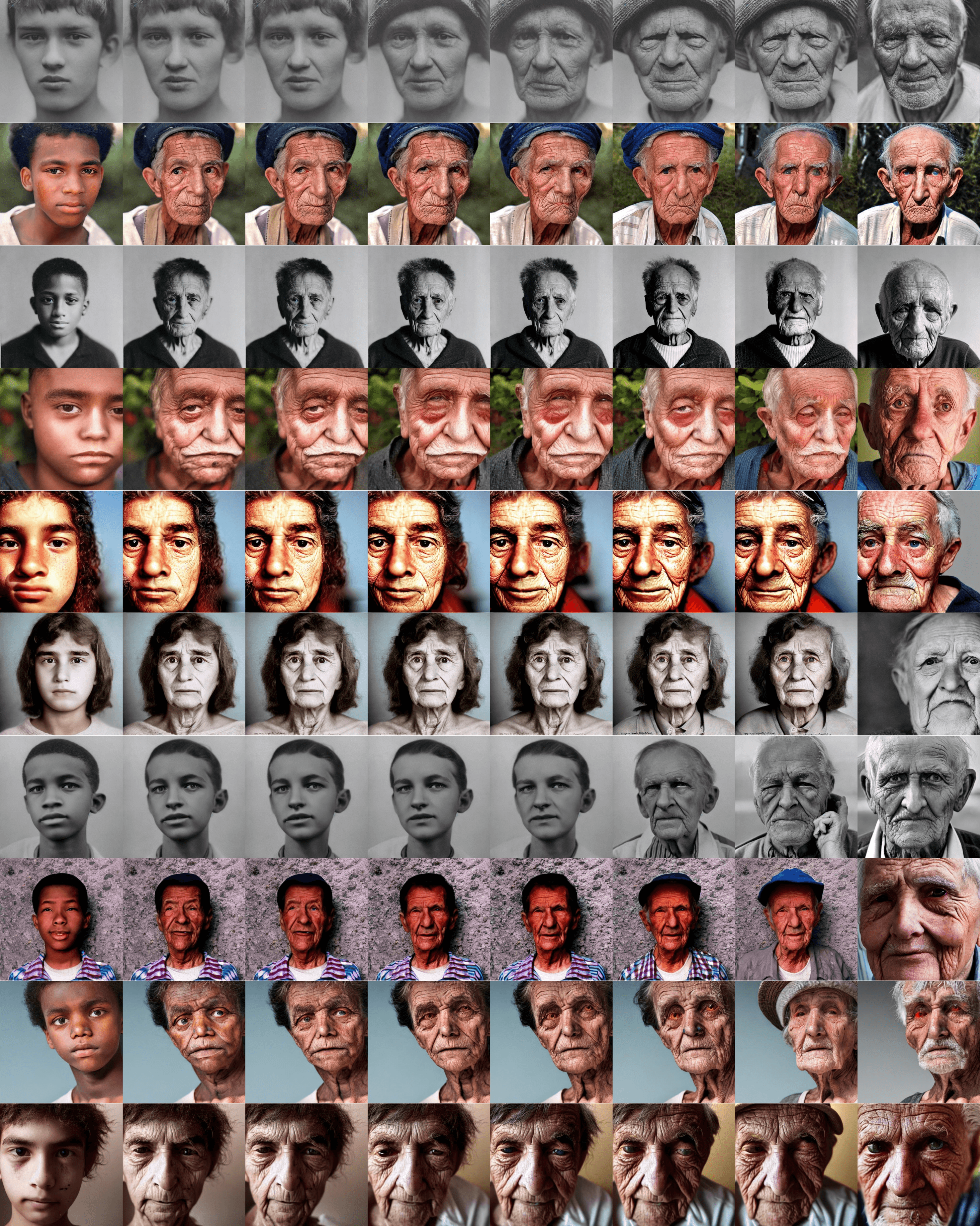}}
    \caption{Text embedding from ``young" to ``old"}
    \label{fig:young_old}
\end{figure*}

\begin{figure*}[!ht]
\centering    
    \subfigure[Mixing]{\includegraphics[width=80mm]{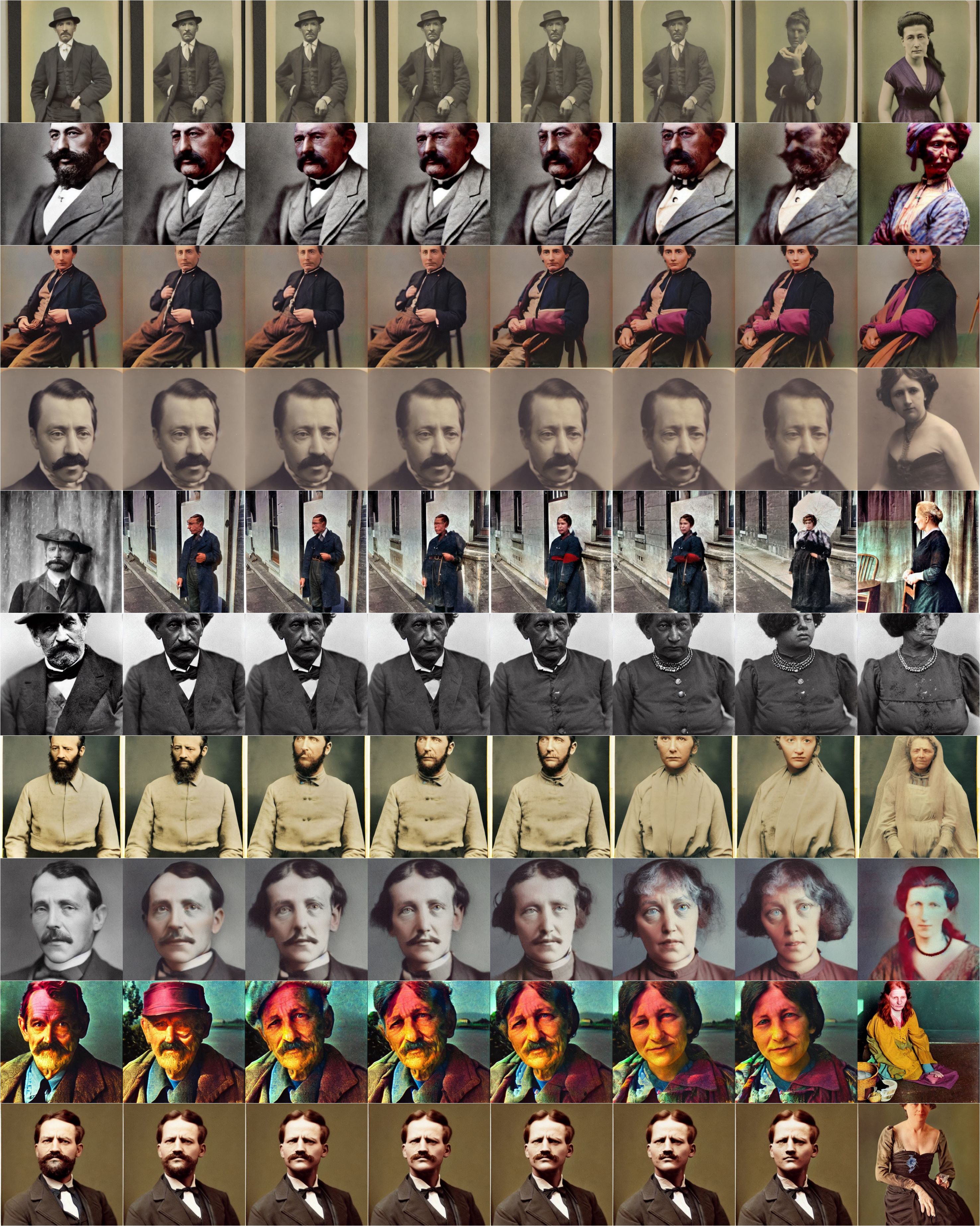}}
    \subfigure[Ours]{\includegraphics[width=80mm]{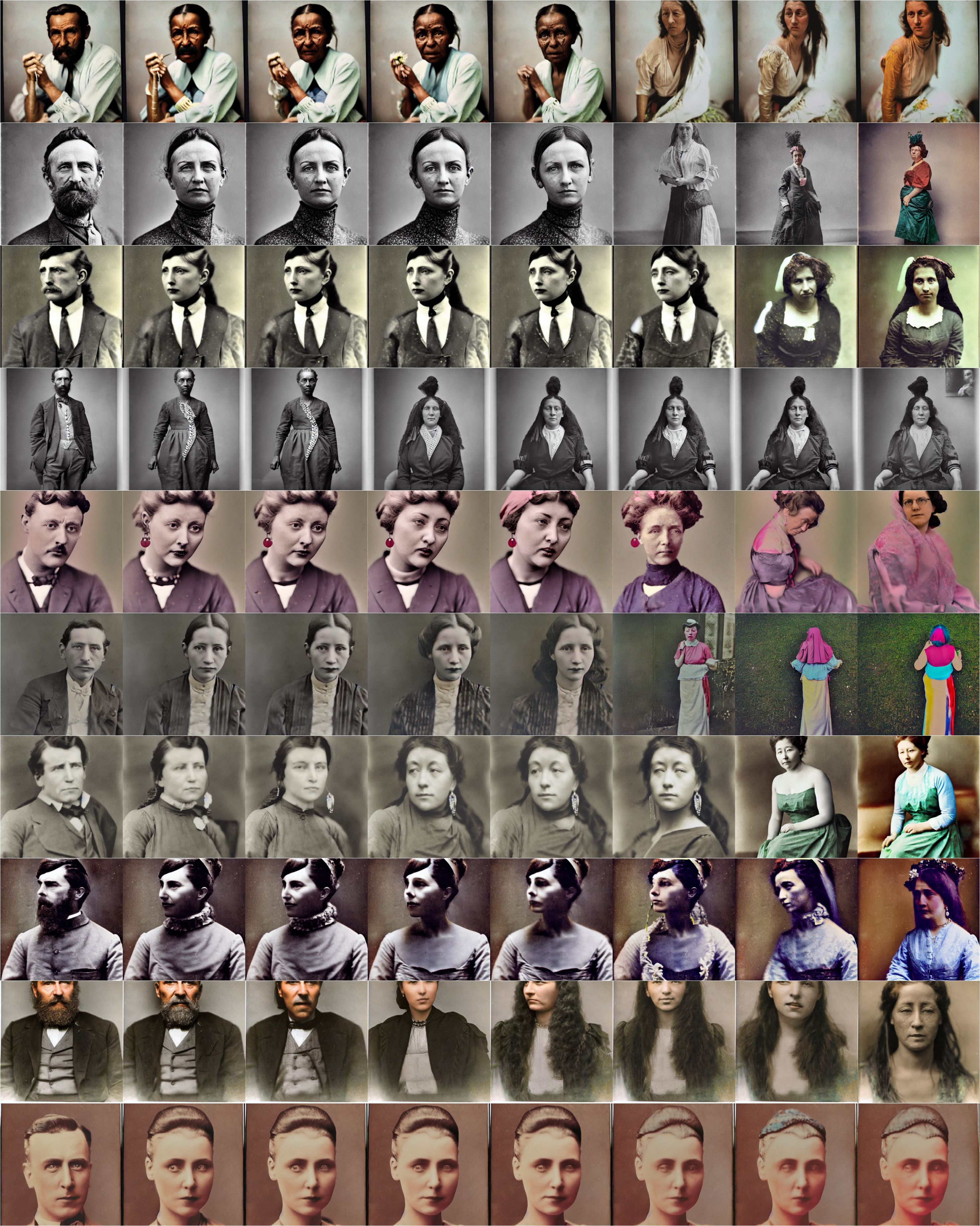}}
    \caption{Text embedding from ``man" to ``woman"}
    \label{fig:man_woman}
\end{figure*}

\end{document}